%% file: main.tex
\documentclass[twoside]{article}

\usepackage[utf8]{inputenc} 
\usepackage[T1]{fontenc}    
\usepackage{hyperref}
\usepackage{url}            
\usepackage{booktabs}       
\usepackage{amsfonts,bbm,amsmath,amsthm}   
\usepackage{nicefrac}       
\usepackage{microtype}
\usepackage{enumitem}
\usepackage{subfigure}
\microtypesetup{protrusion=false}
\usepackage{xcolor}         
\usepackage[pdftex]{graphicx}
\usepackage{algorithm}
\usepackage{algorithmic}
\usepackage{dsfont}
\usepackage{cleveref}
\crefname{assumption}{assumption}{assumptions}
\usepackage{bbm}
\usepackage{mathtools}
\usepackage{caption}
\theoremstyle{definition}
\newtheorem{definition}{Definition}[section]
\newtheorem{lemma}{Lemma}
\newtheorem{theorem}{Theorem}
\newtheorem{remark}{Remark}

\newtheorem{proposition}{Proposition}
\newtheorem{assumption}{Assumption}

\newcommand{\ind}{\mathds{1}}

\newcommand{\mX}{\mathcal{X}}

\newcommand{\NA}{\mathrm{NA}}

\newcommand{\obs}{\mathrm{obs}}
\newcommand{\mis}{\mathrm{mis}}
 
\newcommand{\mask}{\operatorname{mask}}

\usepackage[accepted]{aistats2026}
\usepackage[round]{natbib}

\begin{document}

%
%

\twocolumn[

\aistatstitle{Mask-Conditional Conformal Prediction: Valid Uncertainty For All  Missing Data Mechanisms}
\aistatsauthor{ Jiarong Fan \And Juhyun Park \And  Thi Phuong Thuy Vo \And Nicolas Brunel}

\aistatsaddress{ LaMME, \\ Universit\'e Paris-Saclay,\\ 
Evry \And  ENSIIE, LaMME, \\
Universit\'e Paris-Saclay, \\ Evry  \And ENSIIE, LaMME, \\
Universit\'e  Paris-Saclay, \\
Evry \And ENSIIE, LaMME, \\
Universit\'e  Paris-Saclay, \\ Evry; \\Capgemini Invent France} ]

\begin{abstract}
Conformal prediction (CP) offers a principled framework for uncertainty quantification, but it fails to guarantee mask-conditional coverage when faced with missing covariates. In addressing the heterogeneity induced by various missing patterns, Mask-Conditional Valid (MCV) Coverage has emerged as a more desirable property than Marginal Coverage. In this work, we adapt split CP to handle missing values by proposing a preimpute-mask-then-correct framework that can offer valid coverage.
We show that our method provides guaranteed Marginal Coverage and Mask-Conditional Validity for general missing data mechanisms. A key component of our approach is a reweighted conformal prediction procedure that corrects the prediction sets after distributional imputation (multiple imputation) of the calibration dataset, making our method compatible with standard imputation pipelines.
We derive two algorithms and prove that they achieve both marginal validity and MCV. We evaluate them on synthetic and real-world datasets. 
It reduces significantly the width of prediction intervals w.r.t standard MCV methods, while maintaining the target guarantees.
\end{abstract}

\section{Introduction}

Uncertainty estimation is essential for deploying machine learning models in real-world applications, such as finance, healthcare, and autonomous systems.
A common approach to quantifying uncertainty is through prediction sets, such as intervals in regression or label collections in classification. The main goal of a prediction set is to contain the true label of a sample point with high probability, e.g., 90\%.

To obtain valid prediction sets, Conformal Prediction (CP) has emerged as a widely adopted framework \citep{vovk2005algorithmic}. CP provides a 
model-agnostic and distribution-free approach to constructing prediction sets for any machine learning model \citep{angelopoulos2023conformal}.
Given the random variable of test point $(X, Y) \in \mathcal{X} \times \mathcal{Y}$, where $\mathcal{X}$ contains $d$ features such that $\forall i \in \{1,...,d\}, X_i \in \mathcal{X}_i$ with each $\mathcal{X}_i  $ typically $\mathbb{R}$ or a discrete space, and a miscoverage rate $\alpha$, CP aims to create a set $\hat{C}_{\alpha}(X)$ that contains $Y$  with probability $1 - \alpha$. Split CP \citep{papadopoulos2002inductive}, one of the most used CP techniques, achieves this by leveraging a calibration dataset to evaluate the model’s uncertainty.

While CP is known to be well-adapted for dealing with applications under the exchangeability assumption (weaker than i.i.d.) on the data-generating process, it fails to provide guarantees under perturbations, such as noisy labels, or missing data.  
\citep{einbinder2024label,zaffran2023conformal,zaffran2024predictiveuncertaintyquantificationmissing,feldman2024robust,sesia2024adaptive,zhou2025conformal}. 

Missing data severely degrade model performance and have long been studied in ML and statistics
\citep{emmanuel2021survey,ren2023review}. 
More recent research tends to advocate imputing missing data and eventually choosing the best imputation algorithms to improve the reliability of predictions, such as the impute-then-regress procedures \citep{le2021sa, bertsimas2024simple,lobo2025a}. 

Beyond predictive performance, missing values can affect uncertainty quantification, including CP \citep{zaffran2023conformal,feldman2024robust}. 
To ensure coverage guarantees with missing data, it is crucial to consider the missingness pattern or mask, denoted by the vector $M \in \{0,1\}^d$, where  $M_i = 1$ indicates the feature $X_i$ is missing \footnote{In this paper, for a vector $V$, we use the subscript $V_i$ to denote the $i$-th element of $V$.}. 
\citet{zaffran2023conformal} introduced the concept of Mask-Conditional Validity (MCV, see \Cref{def:marginal}),
emphasizing that missingness induces heterogeneity in the data, causing basic Split CP sets to exhibit under-coverage. If we denote by $P$ the joint distribution of  $(X,Y,M)$, MCV corresponds to a group-wise control adapted to each pattern $m \in \{0,1\}^d$: $P(Y \in \hat{C}_{\alpha}|M=m)\geq 1-\alpha$, which can be achieved by modifying how the calibration dataset is handled via data augmentation methods: CP-MDA-Exact and CP-MDA-Nested \citep{zaffran2023conformal}.
However, the method is limited to MCAR and tends to be conservative in the sense of requiring a large data volume for CP-MDA-Exact, or generating a large prediction set for CP-MDA-Nested in extreme cases, e.g. Figure 12 in \citet{zaffran2023conformal} for data with 40\% missing values.

As aleatoric uncertainty varies significantly with $M$ \citep{zaffran2023conformal},  
MCV is the focus of our work as it is a crucial property to ensure in practice in order to restore risk-controlled and useful prediction sets. To address the shortcomings of previous methods, 
we propose a novel preimpute-mask-then-correct framework built upon the standard practice of imputing first and applying the usual uncertainty quantification tools such as Split CP.
A critical remark is that imputation (only) does not correct for the heterogeneity induced by different patterns of missingness; in order to control the miscoverage, we propose an alignment by missing pattern, illustrated in \Cref{illustration_preimpute}, before applying Split CP.
\begin{figure}[hbt]
\begin{center}
\centerline{\includegraphics[width=0.83\columnwidth]{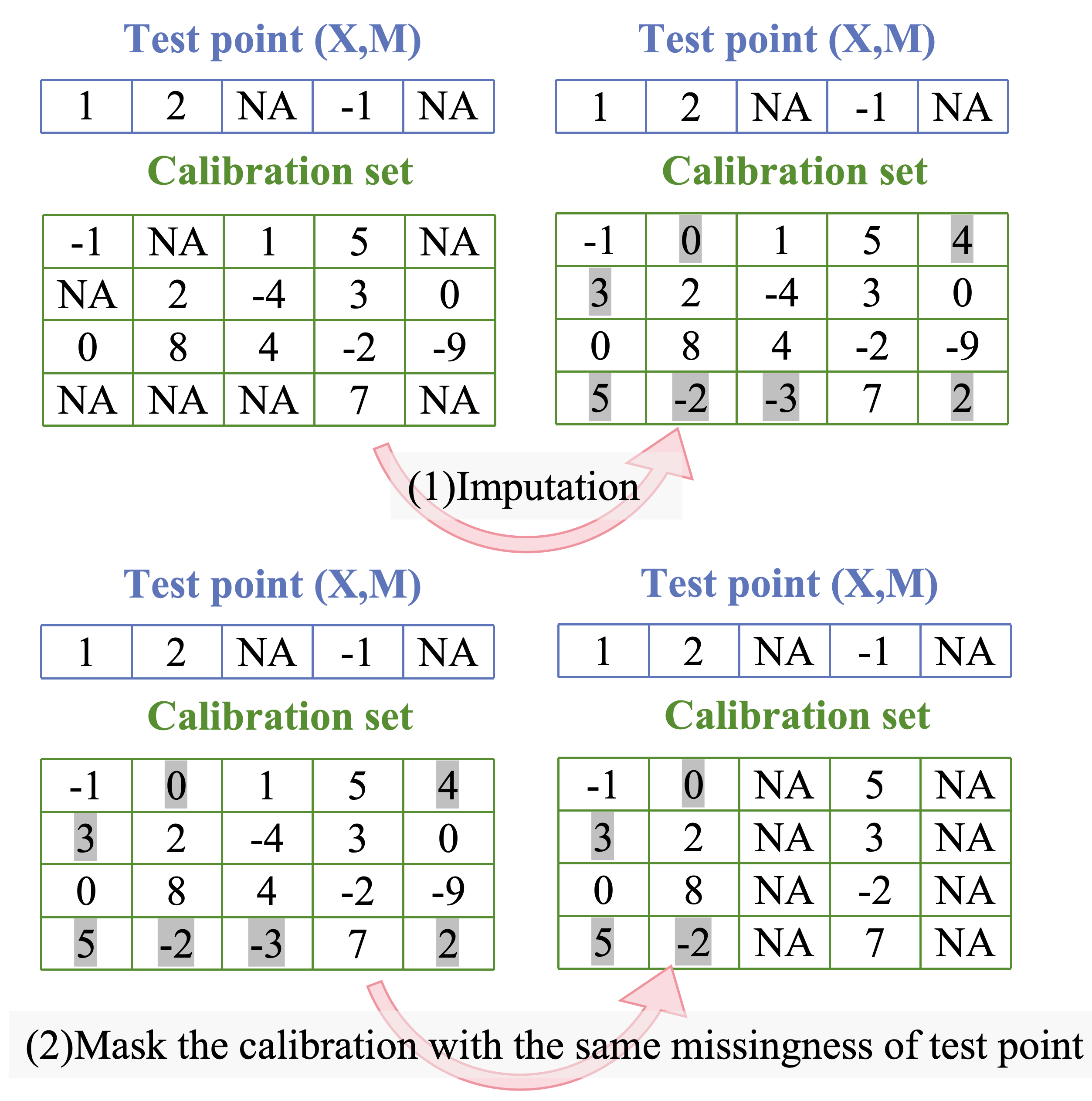}}
\captionsetup{skip=2pt}
\caption{Illustration of the pre-imputation: the calibration data is imputed and then masked by the $M$ of test point such that they share the same pattern. In this way, the split CP can give a prediction set conditional on $M$. But a distribution shift exists. Unlike CP-MDA, our method leverages all available calibration data without modifying the test point.}
\label{illustration_preimpute}
\end{center}
\end{figure}
However, as we show in the experiment in \Cref{section:weight correction}, this approach fails to guarantee MCV due to distribution shifts introduced by imperfect imputation.
We address this issue by generalizing Weighted CP \citep{tibshirani2019conformal} to missing data and introducing two likelihood-based correction mechanisms, ensuring valid intervals even under MAR and MNAR conditions, which are more realistic assumptions \citep{kang2013prevention}. 

\textbf{Contributions.} 

We propose a preimpute-mask-then-correct framework with two novel correction methods, the mask-conditional Weighted CP and the Acceptance Rejection-Corrected (ARC) CP methods, to address the challenges of distribution shift  
while preserving validity of the prediction sets in datasets with missing covariates.
Our ARC-CP method is a novel implementation that uses density-ratio weights to resample a subset, enabling the distribution shift reweighting problem to be solved explicitly.

These approaches ensure both MV and MCV while significantly reducing the width of the prediction interval compared to previous methods, as demonstrated through rigorous theoretical analysis and extensive experiments on synthetic and real-world datasets. The proposed methods guarantee MCV without any assumptions on the missingness mechanism, which, to the best of our knowledge, has not been achieved before. 

We also analyze the effect of inaccurate density ratio estimators on MCV and establish lower bounds on the miscoverage, and empirically show that it remains limited and controlled. 
Since our methodology requires only to impute once the calibration dataset, and does not require imputation of the test point, our methodology can be smoothly integrated into a large number of data pipelines with maintained guarantees, and can be easily used in practice. 

\section{Background}
\label{sec:background}
\textbf{Missing data notation.} 
Given the triplet random variable $Z:=(X,M,Y)\in\mathcal{X}\times \{0,1\}^d \times \mathcal{Y} \sim P$, let $\mathcal{X}_\mathrm{NA}=\prod_{j=1}^d(\mathcal{X}_j \cup \{\mathrm{NA}\})$. We define the mask operator, $\mask:  \mX \times \{0,1\}^d \rightarrow \mX_\mathrm{NA}$, 
\begin{equation*}
    \mask(x,m) \coloneqq (1 - m) \odot x \;+\; m \odot \mathrm{NA},
\end{equation*}
where $\odot$ is the element-wise product. Whenever there is no ambiguity, we use $\widetilde{X}$ to denote $\mask(X, M)$ to simplify notation.

Denote by $\obs(M)$ the set of indices of the observed covariates and by $\mis(M)$ that of the missing covariates. 
For example, if $m=[0,1,0]$, then $\obs(m)=\{1,3\}$ and $X_{\obs(m)}=[X_1,X_3]$.

Following \citet{rubin1976inference}, we recall the standard classification of  missingness mechanisms as MCAR (Missing Completely At Random) when $P(M = m \mid X,Y) = P(M = m)$, MAR (Missing At Random) when $P(M = m \mid X,Y) = P(M = m \mid X_{\obs(M)}) $ and Missing Not at Random (MNAR) when it is neither MAR nor MCAR.

To build a model on missing covariates, one may use imputation or not. Recent works have shown that for almost all deterministic imputation functions $\Phi \in \{\Phi: \mathcal{X}_{\mathrm{NA}}  \rightarrow \mathcal{X} \mid \Phi(\widetilde{X}_{\obs(M)})=X_{\obs(M)}\}$, an impute-then-regress procedure with a consistent learner is Bayes consistent \citep{le2021sa}. 
Of course, imputation is not always necessary. Tree models, such as Gradient Boosting Trees, are designed to handle covariates with missing values. 
For notational simplicity, we use $\hat{f}$ to denote any possible model defined in $\mathcal{X}_{\mathrm{NA}}$.
In the following, we assume $t+n+1$ i.i.d. realizations of the pair of random variables $(\widetilde{X}, Y)$, partitioned into a training set of size $t$, ${\tt Train} = \{(\widetilde{X}_{\text{tr}}^i,{Y}^i_{\text{tr}})\}_{i=1}^t$, a calibration set of size $n$, ${\tt Cal} = \{\widetilde{Z}^i:=(\widetilde{X}^i,Y^i)\}_{i=1}^n$ and one test point.

\textbf{Weighted Split CP.} A popular formulation of CP is Split CP \citep{papadopoulos2002inductive}. Given a pre-fitted model $\hat{f}$ on $t$ training data {\tt Train}
and a holdout set of calibration data {\tt Cal}, 
a nonconformity score $s:\mathcal{X}_{\mathrm{NA}}\times \mathcal{Y} \rightarrow \mathbb{R}$ is calculated on all calibration points.
In practice, $s$ is induced by $\hat{f}$, 
which can be computed either directly on partially observed inputs or after applying an imputation strategy.
Denote by $\mathcal{Q}_\beta(F)$  the $\beta$-level quantile of 
distribution  $F$, i.e., 
$\mathcal{Q}_{\beta}(F)=\inf \{z: F(z) \geq \beta\}$.
Let $F_n:=\{\Sigma_{i=1}^n\delta_{s(\widetilde{Z}^i)}+\delta_{+\infty}\}/(n+1)$, the empirical distribution of $n$ scores, where the additional $\delta_{+\infty}$ is a standard convention in CP to guarantee finite-sample coverage. 
The Split CP prediction set for any test point $\widetilde{X}^{n+1} \in \mathcal{X}_{\mathrm{NA}}$ is defined by
 \begin{equation}
 \label{eq:splitcp}
     \hat{C}_{\alpha}(\widetilde{X}^{n+1}):=\{y \mid 
s(\widetilde{X}^{n+1},y) \leq \mathcal{Q}_{1-\alpha}(F_n)\}.
 \end{equation}
 As long as $\{\tilde{Z}^i\}_{i=1}^{n+1}$ are exchangeable, the prediction set contains $Y^{n+1}$ with probability at least $1-\alpha$ \citep{angelopoulos2021gentle}.

However, if the points are not exchangeable, as in \Cref{illustration_preimpute}, the Split CP has no guarantee of coverage. \citet{tibshirani2019conformal} extended CP to handle non-exchangeable data under covariate shift. 
They proposed reweighting the data using density ratio
between the test and calibration set. 
Other scenarios of non-exchangability include
repeated sampling \citep{sesia2023conformal}, time series \citep{gibbs2024conformal,zaffran2022adaptive} or label shift \citep{si2023pac, podkopaev2021distribution}. 
Recently, \citet{prinster2024conformal} generalize conformal prediction to arbitrary data distributions beyond exchangeability under sequential, feedback-induced covariate shifts. While they significantly extend the scope of Weighted CP, it relies on the knowledge of data generating distribution (possibly factored with graphs), which is rarely available with missing data and imputation mechanisms.

\textbf{Marginal Validity (MV) and Mask-Conditional Validity (MCV).} 
Our task is to construct a prediction set for $Y^{n+1}$ given ${\tt Train}$, ${\tt Cal}$, and $\widetilde{X}^{n+1}$.
Facing missing values, we are particularly interested in properties in \Cref{def:marginal}.
\begin{definition}
\label{def:marginal}
A method producing
intervals $\hat{C}_{\alpha}(\widetilde{X}^{n+1})$ is MV at level $\alpha$ if: 
$$P(Y^{n+1} \in \hat{C}_{\alpha}(\widetilde{X}^{n+1})) \geq 1-\alpha ,$$ and is MCV at level $\alpha$ if, for any $m\in\{0,1\}^d$ such that $P(M=m)>0$, 
$$P(Y^{n+1} \in \hat{C}_{\alpha}(\widetilde{X}^{n+1}) \mid M^{n+1} = m ) \geq 1-\alpha.$$
\end{definition}
The problem of uncertainty quantification under missing values has recently attracted increasing attention. 
Broadly two different approaches have emerged.
First, \citet{liang2024structured, gui2023conformalized, shao2023distribution} addressed the problem through matrix completion, introducing novel CP strategies tailored to matrix data, with strong theoretical guarantees. However, their assumptions-- such as the random observation model, where entries are missing independently according to a Bernoulli distribution (i.e., MCAR)—may not hold in typical supervised learning settings. Moreover, \citet{shao2023distribution} emphasized that matrix prediction and matrix imputation are fundamentally different problems: the former assumes structure across both rows and columns (e.g., low-rank or exchangeability), while the latter deals with fixed covariates and responses, where features are  typically non-exchangeable. Second, \citet{zaffran2023conformal, zaffran2024predictiveuncertaintyquantificationmissing}, which motivated our work, 
addressed it
by introducing two data augmentation techniques (CP-MDA-Exact and CP-MDA-Nested) within the Conformal Prediction with Missing Data Augmentation (CP-MDA) framework, as well as a more flexible generalization, CP-MDA-Nested$^\star$. Their main algorithm is reminded in \Cref{alg:cp_mda_nested}.

\begin{algorithm}[hbt!]
   \caption{\textsc{CP-MDA-Nested$^\star$} \cite{zaffran2024predictiveuncertaintyquantificationmissing}}
   \label{alg:cp_mda_nested}
      \textbf{Input:} ${\tt Cal}=\{(\widetilde{X}^{k}, Y^{k})\}_{k=1}^n$, 
   nonconformity score function $s(\cdot, \cdot)$, significance level $\alpha$, test point $\widetilde{X}^{n+1}$, any subsampled set of calibration indice $\widetilde{Id}_{\text{Cal}} \subseteq \{1,...,n\}$.
   \textbf{Output:} $\widehat{C}^{\text{MDA-Nested$^\star$}}_{n, \alpha}(\widetilde{X}^{n+1})$.
    \begin{align*}
    \text{1:} &\forall k \in \widetilde{Id}_{\tt Cal}\text{, set} \bar{M}^{k} = [\max(M^k_i,M^{n+1}_i),i=1,...,d]\\
       \text{2: }&\widehat{C}_{n, \alpha}^{\text{Nested$^\star$}}(\widetilde{X}^{n+1}(M^{n+1})) :=\Big\{ y :  (1 - \alpha)(1 + \#\widetilde{Id}_{\tt Cal} )\\
       &  > \sum_{k \in \widetilde{Id}_{\tt Cal} } \mathbf{1}\Big[ s(\widetilde{X}^{k}(\bar{M}^{k}), Y^{k})< s(\widetilde{X}^{n+1}(\bar{M}^{k}), y) \Big] \Big\}
   \end{align*}
\end{algorithm}
Under MCAR, CP-MDA-Nested$^{\star}$ has been proven to be MCV at the level $2\alpha$ \citep{zaffran2024predictiveuncertaintyquantificationmissing}. 
The method is CP-MDA-Exact when $\widetilde{Id}_{\tt Cal}=\{k \in  \{1,...,n\}: \sum_{j=1}^d M^k_j = \sum_{j=1}^d M^k_j M^{n+1}_j\}$, i.e., select calibration points having fewer missing values than the test point, and CP-MDA-Nested when $\widetilde{Id}_{\tt Cal}=\{1,...,n\}$. 
Hence, for CP-MDA-Exact, the need of matching the missingness may create a situation without enough calibration points, thus infeasible in small-sample regimes.  
On the other hand, for CP-MDA-Nested, the test point is modified by augmenting the mask with additional missing values, which inevitably increases its conservativeness and generates larger prediction intervals. The CP-MDA-Nested$^\star$ with a more flexible target missingness pattern is a compromise between the two methods, but how to choose the best $\widetilde{ID}_{\tt Cal}$ remains a question.

\textbf{Relationship to Weighted CP and Shift Correction.} Beyond matrix completion and CP-MDA methods, our work is also connected to a broader research that uses importance-weighted learning to correct for distributional shifts, as in \citet{bickel2009discriminative, zhou2023domain}. While previous works have largely focused on improving importance weight estimation under covariate shift, our focus is distinct: we study how such estimators can be integrated within conformal prediction to maintain mask-conditional coverage under missing data mechanisms. 
Thus, our main novelty lies in the preimpute-mask-then-correct framework adapted to missingness and we emphasize the role of imputation and its impact on calibration validity. In addition, we provide an evaluation of the consequences of inaccurate estimators, quantifying how poor estimation affects the miscoverage of the prediction sets. 

\section{Methodology}
\subsection{Preimputation of calibration dataset}

As shown in \Cref{illustration_preimpute}, 
the preimputation ensures that Split CP can be applied. Since the response $Y$ is observed in the calibration set, it is natural to use this for imputation, in order to improve the reconstruction of missing covariates. Hence, we consider a potentially more general imputation that uses both $X_{\text{obs}(M)}$ and $Y$, instead of the usual ones that rely only on $X_{\text{obs}(M)}$. The latter case is covered by our results, but 
we find in our experiments that the general approach provides better imputations. A critical point in our work is the use of distributional (or multiple) imputation, which has been shown to provide good results and desirable properties \citep{naf2024good}, such as MICE. 

Hence, we consider only  distributional imputation strategies for the calibration data set.
Formally, starting from the partially observed data $(\widetilde{X}, Y)$, we derive a complete version $(\hat{X}, Y)$, where law of $\hat{X}\vert (\widetilde{X}, Y)$ follows a certain distribution supported in $\mathcal{X}$. Let $Q$ denote the joint distribution of $\hat{Z}:=(\hat{X},Y).$  
Applying this procedure to each calibration point, we construct an imputed calibration dataset $\{\hat{Z}^i:=(\hat{X}^i, Y^i)\}_{i=1}^n$ and an imputed test point $\hat{Z}^{n+1}:=(\hat{X}^{n+1}, Y^{n+1})$. It should be noted that, $\hat{X}^i$ being the imputed value for $\widetilde{X}^i$, we have $\hat{X}^i_{\obs(M^i)}= X^i_{\obs(M^i)}$. 
Such imputations are commonly used in practice, such as Bayesian Linear Ridge in Scikit-learn \citep{pedregosa2011scikit}.
The formal definition is in Appendix~\ref{appendix:preimpute}. 
For example, if we use a Bayesian ridge regression for imputation, then $\hat{X}_{\mis(M)}|(\widetilde{X},Y)\sim \mathcal{N}\big(\hat{\mu}(\widetilde{X},Y), \hat{\Sigma}(\widetilde{X},Y)\big)$, composed with estimated conditional mean $\hat{\mu}(\widetilde{X},Y)$ and covariance matrix $\hat{\Sigma}(\widetilde{X},Y)$. 

Throughout the article, the following assumptions are made.
Here, for notational simplicity, we slightly abuse notations: for $\mathcal{I}$ a set of indices (e.g. $\mathcal{I}=\obs (m)$), ${Z}^i_{\mathcal{I}}:=({X}^i_{\mathcal{I}}, Y^i)$, $\hat{Z}^i_{\mathcal{I}}:=(\hat{X}^i_{\mathcal{I}}, Y^i)$ and $\widetilde{Z}^i_{\mathcal{I}}:=(\widetilde{X}^i_{\mathcal{I}}, Y^i)$.
\begin{assumption}
\label{assump:1}
~
\begin{enumerate}[label=(A\arabic*), itemsep=0pt, topsep=0pt]
    \item The 
    calibration and test
    data $\{\widetilde{Z}^i\}_{i=1}^{n+1}$ are i.i.d. 
    \item The imputer is trained on the training set and applied in the same way to all points; hence, the imputed dataset $\{\hat{Z}^i\}_{i=1}^{n+1}$ is also i.i.d.
    \item The predictive model is 
    fitted only on the training set.
\end{enumerate}
\end{assumption}
These assumptions are standard in the CP literature and are generally satisfied in practice. 

\begin{remark}
\label{lem:iid}
Based on \Cref{assump:1}, all points imputed are i.i.d. Thus, any prediction set using Split CP on the imputed calibration and test datasets is MV, as the imputation preserves exchangeability. This result generalizes Lemma 3.2 of \citet{zaffran2023conformal}, which was limited to the deterministic imputation case.
We additionally show that ``impute-then-split CP" ensures marginal validity for distributional imputation.
\end{remark}
For any mask $m \in \{0,1\}^d$, we denote $ \mathcal{X}_{\text{obs}(m)}:=\prod_{i \in \text{obs}(m)} \mathcal{X}_i$ and  for any measurable subset $A \subseteq \mathcal{X}_{\text{obs}(m)} \times\mathcal{Y}$, we define
$Q_m(A) := Q(\hat{Z}_{\text{obs}(m)} \in A)$ , and the original conditional measure as $P_m(A) := P(Z_{\text{obs}(m)} \in A \mid M=m)$. After pre-imputation, conditionally on $M^{n+1}$, calibration points follow
$\hat{Z}^i_{\obs(M^{n+1})} \sim Q_{M^{n+1}}, \;i=1,...,n$, while the test point follows the true distribution $Z^{n+1}_{\obs(M^{n+1})} \sim P_{M^{n+1}}.$ Generally, $Q_{M^{n+1}} \neq P_{M^{n+1}}$, resulting in a mask-conditional distribution shift. This discrepancy limits the direct application of Split CP.

We introduce now \Cref{prop:bound} that gives a first insight into the impact of imputation on the mask-conditional coverage.
\begin{proposition}
\label{prop:bound}
    Under \Cref{assump:1}, if we apply Split CP to the test point $\widetilde{X}^{n+1}$, with masked calibration dataset $\{\mask(\hat{X}^i,M^{n+1}),Y^i\}_{i=1}^{n}$, then for almost every $m \in \{0,1\}^d$ such that $P(M=m) > 0$, 
    the corresponding prediction set $\hat{C}_{\alpha}$ defined in \cref{eq:splitcp} satisfies:
\begin{align*}
    &P(Y^{n+1}\in \hat{C}_{\alpha}(\widetilde{X}^{n+1})| M^{n+1}=m) \geq 1-\alpha-\\
    &\left(d_{TV}(P,Q)+\mathbb{E}_{P}\left[\left|\frac{dP(X,Y|M=m)}{dP(X,Y)}-1\right|\right]\right)
\end{align*}
In addition, if $P$ is MCAR:
$$P(Y^{n+1}\in \hat{C}_{\alpha}(X^{n+1})| M^{n+1}=m) \geq 1-\alpha-d_{TV}(P,Q)$$
\end{proposition}
See Appendix~\ref{proof:prop:bound} for the proof. The bound decomposes the mask-conditional miscoverage into two components: the error due to imperfect imputation and the intrinsic distribution shift caused by mask-dependent sampling. This first result indicates that under MAR or MNAR conditions, when $P(X, Y| M=M^{n+1})\ne P(X, Y)$, valid coverage may not be guaranteed even when the imputation is perfect, i.e. $P=Q$. 
The second result highlights the importance of the quality of pre-imputation in ensuring valid coverage under MCAR. Even if our lower bounds may be loose, our results apparently contradict the results of~\citet{le2021sa}, which focus on consistency and asymptotic Bayes-optimality, stating that ``imputation doesn't matter for regression task''. While these results are first-order, our method highlights the significant impact of pre-imputation on second-order properties, such as uncertainty and MCV. A detailed discussion on pre-imputation quality is provided in Appendix~\ref{appendix:preimpute}.

\subsection{Weighted conformal prediction under mask-conditional distribution shift}

To solve the distribution shift caused by imputation, we introduce \Cref{assump:absolute_continuous}. 
    \begin{assumption}
\label{assump:absolute_continuous}
    The joint distribution of $(X,Y) \sim P$ is absolutely continuous with respect to $(\hat{X},Y) \sim Q$, and we denote $\omega(x,y) := \frac{dP}{dQ}(x,y)$.
\end{assumption}
This assumption is satisfied in many common scenarios. For instance, if the missing covariates are continuously distributed and we use a distributional imputation method (e.g., MICE), then the imputed distribution is absolutely continuous with respect to the original one. If the missing covariates are discrete or contain mass points, the condition still holds as long as the imputation assigns nonzero probability mass to those same points. 
In practice, this property can be easily ensured, for example, by using distributional imputers or simply by adding a small Gaussian perturbation to the imputed values, which guarantees that the imputed distribution covers the full support of the original one.

Under \Cref{assump:absolute_continuous}, for any mask $m \in \{0,1\}^d$ such that $P(M=m)>0$, the conditional distribution $P_m$ is absolutely continuous w.r.t. $Q_m$ with likelihood ratio $\omega_m(x_{\text{obs}(m)},y) := \frac{dP_m}{dQ_m}(x_{\text{obs}(m)},y)$, with convention $\omega_m=0$ if $Q_m=0$. 
Denote for all $i \in \{1,...,n\}$, $\omega_m^i:=\omega_m(\hat{X}^i _{\text{obs}(m)},Y^i)$, and define the weights functions for  all $ i \in \{1,2,...,n\}$, $W^{i}_m:\mathcal{X}_{\obs(m)} \times \mathcal{Y} \;\rightarrow\; \mathbb{R}$:
\[
W^{i}_m(x_{\text{obs}(m)},y):=\frac{\omega_{m}^{i}}{\sum_{j=1}^{n} \omega^j_m+\omega_m(x_{\text{obs}(m)},y)},
\]
\[
\text{and } W^{n+1}_m(x_{\text{obs}(m)},y):=\frac{\omega_m(x_{\text{obs}(m)},y)}{\sum_{j=1}^{n} \omega^j_m+\omega_m(x_{\text{obs}(m)},y)}.
\]
For any nonconformity score $s:\mathcal{X}_{\mathrm{NA}} \times \mathcal{Y} \rightarrow \mathbb{R}$, define on the same domain the score for imputed-then-masked calibration points: 
$S^{i}_{m}:=s(\text{mask}(\hat{X}^i,m),Y^i)$. 
\begin{theorem}
\label{thm:weighted_cp_mcv}
    Under \Cref{assump:1,assump:absolute_continuous},  
define the weighted split CP interval conditional on $M^{n+1}=m$:
\begin{align}
&\widehat{C}_{\alpha}^W(\widetilde{X}^{n+1})=\{y : s(\widetilde{X}^{n+1},y) \label{eq:weighted_prediction_set} \leq \mathcal{Q}_{1-\alpha}(\\
&\sum_{i=1}^{n} W^{i}_m(\widetilde{X}^{n+1}_{\text{obs}(m)},y) \delta_{S^{i}_{m}}
+W_m^{n+1}(\widetilde{X}^{n+1}_{\text{obs}(m)},y) \delta_{+\infty})
\} \nonumber
\end{align}
Then for any missingness mechanism (MCAR, MAR, or MNAR):
$P(Y^{n+1} \in \widehat{C}_{\alpha}^W(\widetilde{X}^{n+1})|M^{n+1}=m) \geq 1-\alpha .$
\end{theorem}
See Appendix ~\ref{proof:thm:weighted_cp_mcv} for proof. The key insight of \Cref{thm:weighted_cp_mcv} is that when imputation introduces a mask-conditional distribution shift, 
the likelihood ratio $\omega_m$ adjusts the quantile and ensures valid coverage. Unlike \citet{tibshirani2019conformal} which addresses only covariate shifts with explicit prediction intervals, our method requires $y$-dependent weights for each calibration point.
The formal \Cref{alg:weighted_cif} is presented in Appendix~\ref{appendix:proof} and our practical strategies for implementation are given in Appendix~\ref{appendix:weighted_cp_implementation}.
The result and algorithm of \Cref{thm:weighted_cp_mcv} still apply if  $\omega_m(x_{\text{obs}(m)},y)$ are known modulo a multiplicative constant. Nevertheless, it is still computationally expensive and  we propose below a more efficient algorithm.
\subsection{ARC conformal prediction under mask-conditional distribution shift}
Generally, one would expect the pre-imputation to have at least some ability to reconstruct the original distribution, such that the gap between imputed data (proposal distribution $Q_m$) and complete distribution (target distribution $P_m$) can be reduced. Therefore, this is a well-suited scenario for drawing observations from $Q_m$ by using acceptance-rejection (AR) to target $P_m$. Suppose $\omega^i_m$ is bounded by a constant $K > 0$. 
For each point, we sample $U^{i} \sim U(0, 1)$ and retain the point if and only if:
$U^{i}<{\omega^i_m(\hat{X}^i_{\text{obs}(m)},Y^i)}/{K}$.
The formal \Cref{alg:ar_cif} is presented in Appendix~\ref{appendix:proof}.
\Cref{alg:ar_cif} outputs accepted $\{(\hat{X}^j,{Y}^j)\}_{j\in \mathcal{I}_{\tt AR}}$, which we mask with $m$ to create a new calibration dataset $\widetilde{\tt Cal}_m=\{(\text{mask}(\hat{X}^j,m),{Y}^j)\}_{j\in \mathcal{I}_{\tt AR}}$. \Cref{thm:ar_correction} shows that 
$\widetilde{\tt Cal}_{M^{n+1}}$ has distribution $P_{M^{n+1}}$ such that we can apply standard split CP calibration to get an MCV prediction set on test point $\widetilde{X}^{n+1}$.
\begin{theorem}
\label{thm:ar_correction} For any $m \in \{0,1\}^d$, suppose $\omega^i_m$ is bounded by a constant $K > 0$. 
    We generate $\widetilde{\tt Cal}_m$ with index $\mathcal{I}_{\tt AR}$ selected by \Cref{alg:ar_cif}. Then: 
      (a) The points in $\widetilde{\tt Cal}_m$ are i.i.d. samples from $P_{m}$; 
     (b) The split prediction set $\widehat{C}^{\tt AR}_{\alpha}$ (see \Cref{alg:ar_cif}) satisfies 
    \[
    P(Y^{n+1} \in \widehat{C}^{\tt AR}_{\alpha}(\widetilde{X}^{n+1}) \mid M^{n+1}=m) \geq 1-\alpha.
    \]
\end{theorem}
See Appendix~\ref{proof:thm:ar_correction} for proof and the detailed ARC-CP algorithm. Theoretically, if \Cref{assump:absolute_continuous} is violated, the likelihood ratio may not be well-defined or bounded, compromising our methods' validity. This situation occurs when certain observable points cannot be generated by imputation (see Appendix~\ref{appendix:preimpute} for discussion), highlighting the importance of reliable imputations
and the benefit of having close distributions (with the same support) to ensure a high acceptance rate.
\subsection{Estimation of likelihood ratio $\omega_m$}
Similar to \citet{tibshirani2019conformal}, we propose the following \Cref{alg:likelihood_ratio_estimation} for training an estimator for the likelihood ratio $\omega_m$. We provide here a sketch of the proof for its validity. The idea is to label the randomly generated data $\tilde{x}$ ``1''  if it does not contain any imputed value (i.e. $(X_{\mathrm{obs}(m)}, Y) \sim P_m$) or ``0'' if it does (i.e. $(X_{\obs(m)}, Y) \sim Q_m$). We train a binary classifier such that we estimate the probability $P(C=1 \mid \widetilde{X}=\tilde{x}, Y=y)$ with $\hat{p}(\tilde{x}, y)$. By Bayes' theorem, we get $\frac{P(C=1 \mid \widetilde{X}=\tilde{x}, Y=y)}{P(C=0 \mid \widetilde{X}=\tilde{x}, Y=y)} = \beta ({d P_m}/{d Q_m})(x_{\mathrm{obs}(m)}, y)$, with the constant $\beta = {P(C=1)}/{P(C=0)}$. The estimation of $\omega_m$ uses solely the calibration dataset, and does not depend on the test value $(x,y)$. 

\begin{algorithm}[htb!]
\caption{{Estimation of Likelihood Ratio}}
\label{alg:likelihood_ratio_estimation}
\textbf{Input:} Points $\{(x^i, m^i, y^i)\}_{i=1}^t$ sampled from imputed set $(\hat{X}, M, Y)$; number of random masks $q \in \mathbb{N^{\star}}$. \textbf{Output:} Estimator of $\omega_m(x_{\obs(m)}, y)$.
\begin{algorithmic}[1]
\STATE For each $(x^i, m^i, y^i)$, generate $q$ random masks $\{\check{m}_j^i\}_{j=1}^q$.
\STATE Compute $c_j^i = \ind \{m^i = \check{m}^{i}_{j}\}$ for all $i,j$, and construct dataset $\{(x^i, \check{m}_j^i, y^i, c_{j}^i)\}$, and rebalance it to have same proportions for two classes $c=1$ and $c=0$. 
\STATE Train a probabilistic classifier $\hat{p}$ (e.g., histogram-based GBDT) to predict $c_j^i$ from $(\mask({x}^i{,\check{m}_j^i)}, y^i)$.
\STATE $\hat{\omega}_m:={\hat{p}(\tilde{x}(m), y)}/({1 - \hat{p}}(\tilde{x}(m), y))$ estimates $\omega_m$.
\end{algorithmic}
\end{algorithm}

\subsection{Impact of ratio estimation on coverage}
The accuracy of the estimator may impact the coverage of Weighted CP. 
This can be the case when it is difficult to estimate, for instance in high dimensions. 
Therefore, we provide below a theoretical analysis of the potential miscoverage, and complement it with an empirical study in Appendix~\ref{appendix:estimator quality} for completeness.

Let $\hat{\omega} _m(\cdot)$ be our estimator of the density ratio and we define $\tilde{\omega} _m(\cdot):={\hat{\omega} _m}(\cdot)/{\mathbb{E}_Q[\hat{\omega} _m(\hat{X}_{\obs(m)},Y)]}$. 
     If we denote $d\tilde{P}_m:=\tilde{\omega}_m(\hat{X}_{\obs (m)},Y)\cdot dQ_m$, then, under Assumption \ref{assump:absolute_continuous}, $\tilde{P}_m$ is a well defined probability measure.
\begin{proposition}
\label{prop:imperfect_densityratio}
     Under Assumption \ref{assump:absolute_continuous}, 
     for any mask $m$, if we use the estimated likelihood ratio $\hat{\omega}$ in (\ref{eq:weighted_prediction_set}), denoted by $\widehat{C}_\alpha^{\widehat W}$, then we have 
     \begin{align*}
          &P(Y^{n+1} \in \widehat{C} _{\alpha}^{\widehat W}(\widetilde{X}^{n+1} )\mid M^{n+1}=m)  \\
          & \geq(1-\alpha)  -d_{TV}(\tilde{P}_m,P_m). \notag
     \end{align*}
\end{proposition}
To gauge the loss, we propose an implementable lower bound for the coverage below.
\begin{proposition}
\label{prop:estimablebound}
Suppose we have an extra dataset $\{(\check{X}^k,\check{Y}^k,\check{M}^k)\}_{k=1}^{J}$, with i.i.d. points from $P$. Using \Cref{alg:generate_P}, generate two datasets $\{Z^i_P\}_{i=1}^{l_1} \sim P_m$ and $\{Z_{\tilde{P}}^i\}_{i=1}^{l_2} \sim \tilde{P}_m$, and set $l=\min(l_1,l_2)$. Form the balanced classification sample by adding a label $C$, $Z_c := \{(Z^i_{P},1)\}_{i=1}^{l} \cup \{(Z^i_{\tilde P},0)\}_{i=1}^{l}$.
Let $\hat g_m\in\arg\min_{g\in\mathcal G_m}\hat R_\gamma(g_m)$ be an empirical risk minimization over $\mathcal G_m$
with respect to the $0$–$1$ loss $\gamma$ on $Z_c$ defined in Appendix~\ref{sec:regularity}, and denote by
$\hat R_\gamma(\cdot)$ the empirical risk on $Z_c$.
Under regularity \Cref{ass:ERM01,ass:M,ass:realizable,ass:stability,ass:VC}, for any $\epsilon>0$, we have with probability at least $(1-\exp{(-\frac{\epsilon^2}{2l\beta_l^2})}-\exp{(-\frac{\epsilon^2 l}{2})})$ that:
\begin{align*}
    &P(Y^{n+1}\in \hat{C}^{\widetilde W}_{\alpha}(\tilde{X}^{n+1})\mid  M^{n+1}=m) \notag \\
    > &1-\alpha-(1-2\, \hat{R}_\gamma (\hat{g}_m))-\epsilon-K\sqrt{\frac{V_{\mathcal{G}_m}}{l}},
\end{align*}
with $\beta_l,V_{\mathcal{G}},K$ defined in Appendix~\ref{sec:regularity}. 
\end{proposition}

See Appendix~\ref{proof:prop3} for the proof. With an imperfect estimator of the likelihood ratio, \Cref{prop:estimablebound} yields an estimable lower bound for coverage, where the key interpretable term is $\tfrac{1}{2}-\hat R_\gamma(\hat g_m)$. Intuitively, if the estimator is nearly accurate, then the discriminator $\hat g_m$ cannot distinguish between $P_m$ and $\tilde P_m$, resulting in $\hat R_\gamma(\hat g_m)\approx \tfrac{1}{2}$. Conversely, if the estimator is poor, $\hat g_m$ can easily separate and lead to $\hat R_\gamma(\hat g_m)\approx 0$, which weakens the coverage guarantee. Note that this requires estimating $g_m$ for each mask $m$, which may seem costly. 
However, Section~\ref{sec:experiments} confirms that even with approximate estimators, our method achieves near-nominal coverage across diverse missingness settings—validating the practical robustness of our method.
\section{Experiments}
\label{sec:experiments}
The primary goals of numerical studies are: (1) to validate the MCV property of the proposed methods under missingness scenarios, (2) to compare the interval width and coverage performance with CP-MDA methods, (3) to demonstrate the necessity and utility of distribution shift corrections, (4) to study the impact of imperfect estimators of likelihood ratio and evaluate their miscoverage.

For simplicity, we consider $\mathcal{X}=\mathbb{R}^d$ and $\mathcal{Y}=\mathbb{R}$. We use the  nonconformity score $s$ defined by Conformalized Quantile Regression method \citep{romano2019conformalized}. Specifically, the model $\hat{f}$ is composed of an iterative ridge imputation and a pair of Quantile Regression Forest \citep{meinshausen2006quantile} models $(\hat{f}_{low},\hat{f}_{high}),$ which are both learned on an independent training set. $s$ is thus defined by $s(\tilde{x}, y)=\max ({\hat{f}}_{low}(\tilde{x})-y, y-{\hat{f}}_{high}(\tilde{x}))$.   Iterative Bayesian Ridge (MICE in Scikit-learn with sample\_posterior being True) is used to pre-impute the calibration. 
The target $1-\alpha$ is fixed to $90\%$, and the likelihood ratio $\omega_m$ is estimated using a Histogram-based Gradient Boosting Classification Tree.

We present results on synthetic data
and real UCI datasets 
under MCAR, MAR, and MNAR \footnote{The code is available on \url{https://github.com/jrfan/Mask-Conditional-Conformal-Prediction}}.
Missing values are generated accordingly. For each task, we generate different datasets and repeat experiments to compute a confidence interval for each mask's average coverage and interval width.
All results show that our methods significantly reduce the width of prediction intervals,
while also guaranteeing the MCV.
\subsection{Experiments on synthetic dataset}
\label{Section:exp on synthetic mcar}
We follow \citet{zaffran2023conformal,zaffran2024predictiveuncertaintyquantificationmissing}
  to generate synthetic datasets, as the oracle width of prediction intervals can be derived and compared.
 The data follows $Y=\beta^{\top} X+\varepsilon$, $\beta \in \mathbb{R}^{d}$, $ X \sim \mathcal{N}(\mu, \Sigma)$ $\varepsilon \sim \mathcal{N}(0,1) {\perp\!\!\!\perp} (X, M)$ where  $d=10, \mu=(1, \cdots, 1)^{\top}$, $\Sigma=\rho(1, \cdots, 1)^{\top}(1, \cdots, 1)+(1-\rho) I_{d}$ with $\rho=0.8$ and $\beta=(1,2,-1,3,-0.5,-1,0.3,1.7,0.4,-0.3)^{\top}$. 
 To better visualize the differences between methods, we use a relatively high missingness probability of 50\% per feature. We also report results for 20\% missingness in Appendix~\ref{appendix:comp} to provide another perspective. The oracle interval width, defined as the minimum width of MCV interval knowing $X_M$, is calculated by Proposition 4.1 in \citet{zaffran2023conformal}. In our experiments, 500 training data and 100 calibration data are used. 100 test data are generated for 1023 types of missingness patterns ($2^{10}-1$, the completely missing pattern is excluded). Each group is used separately to evaluate the missingness-wise coverage and interval width. We repeat each experiment 500 times.

\begin{figure}[hbt]
\centering
\includegraphics[width=.92\columnwidth]{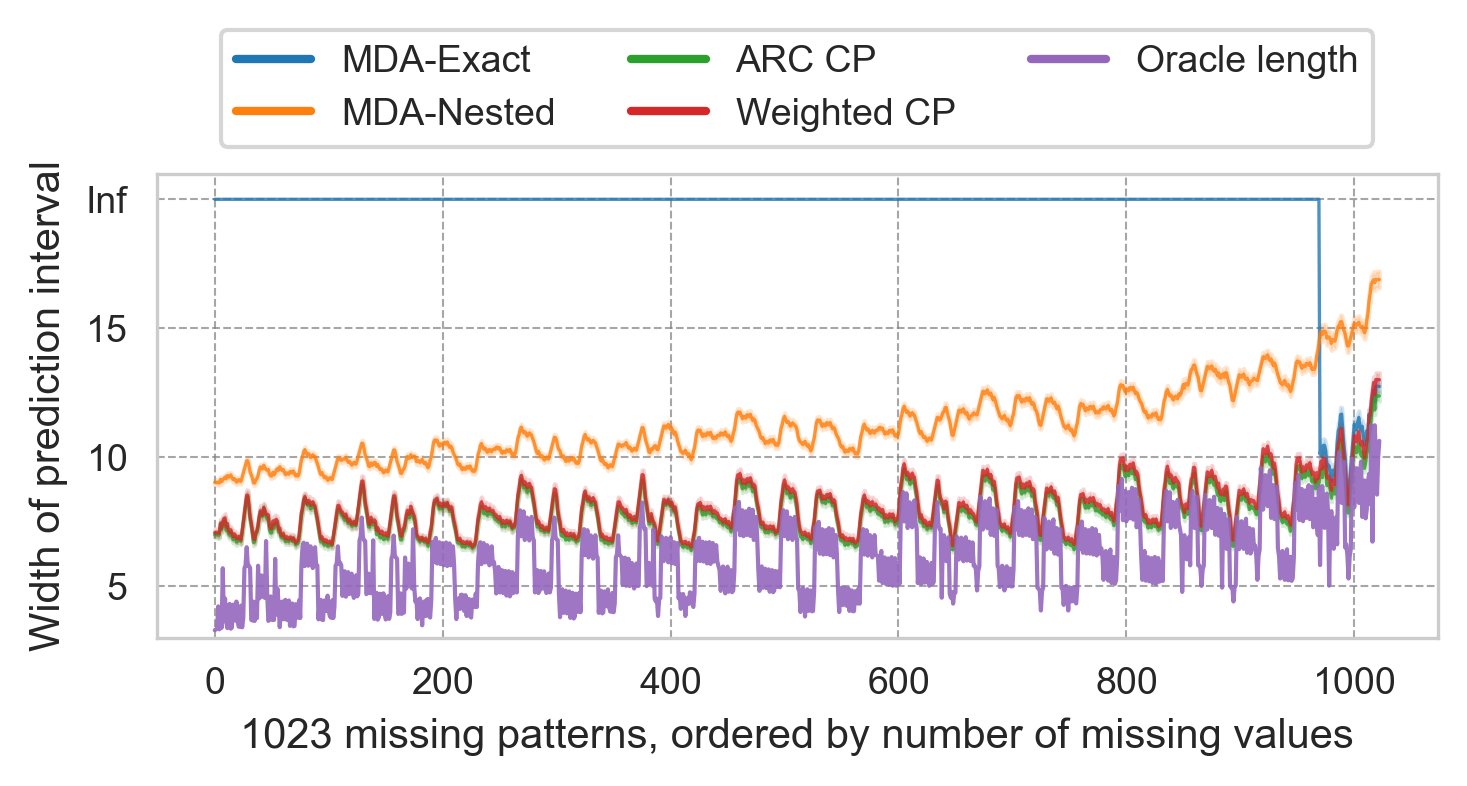}
\caption{Synthetic dataset on MCAR: prediction interval width across 1023 masks (mean and 95\% CI).}
\label{synthetic_width}
\end{figure}

\begin{figure}[hbt]
\centering
\includegraphics[width=.92\columnwidth]{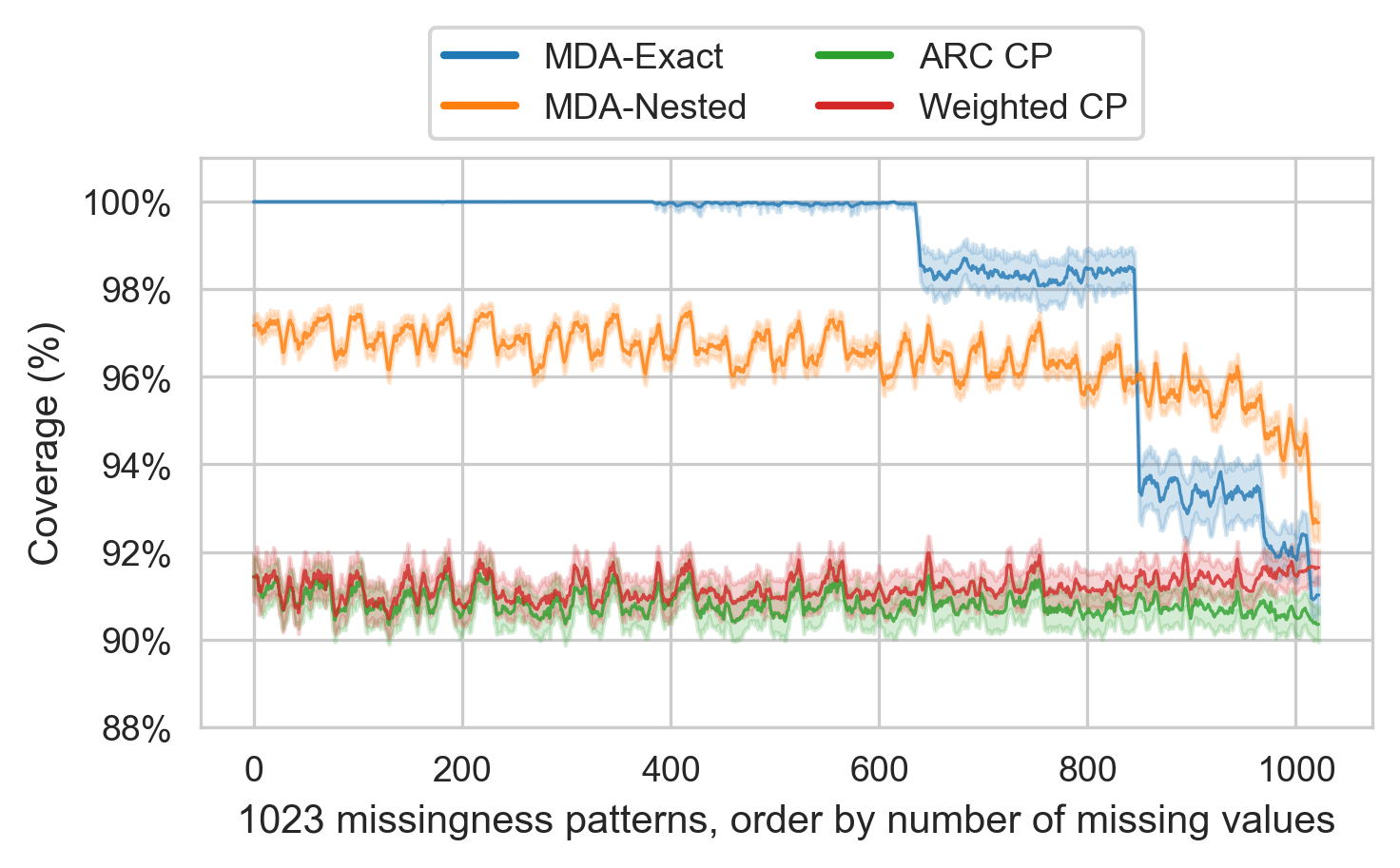}
\caption{Synthetic dataset on MCAR: prediction coverage across 1023 masks (mean and 95\% CI).}
\label{synthetic_coverage}
\end{figure}

\begin{table*}[!t]
\centering
\caption{Worst-case coverage (95\% CI) and average prediction interval width over 31 masks.}
\label{tbl:mar_mnar_results}
\begin{tabular}{lcccc}
\toprule
Mechanism & MDA-Exact & MDA-Nested & Weighted CP & ARC CP \\
\midrule
MAR  & 87.8--89.7\% / 5.84 & 87.9--89.8\% / 5.67 & 89.0--91.2\% / 5.79 & 88.3--90.2\% / 5.64 \\
MNAR & 87.8--89.8\% / 5.74 & 87.8--89.8\% / 5.64 & 89.1--91.0\% / 5.79 & 88.5--90.3\% / 5.60 \\
\bottomrule
\end{tabular}
\end{table*}

\begin{table*}[t]
\centering
\caption{100-dimensional synthetic setting: average mask-conditional coverage and interval width,
grouped by oracle interval width $L_{\text{oracle}}(m)$. Results of coverage and width are averaged over 1000 masks with 50 replications.}
\label{tbl:100d_oracle}
\begin{tabular}{lcccccc}
\toprule
Bin of $L_{\text{oracle}}(m)$
& ARC-CP cov & ARC-CP wid
& WCP cov & WCP wid
& Nested cov & Nested wid \\
\midrule
$[17.74, 20.39]$ & 0.911 & 27.92 & 0.884 & 27.14 & 0.926 & 29.58 \\
$[20.39, 21.16]$ & 0.910 & 27.99 & 0.885 & 27.20 & 0.927 & 29.76 \\
$[21.16, 21.57]$ & 0.910 & 28.04 & 0.885 & 27.16 & 0.928 & 29.90 \\
$[21.57, 22.02]$ & 0.911 & 28.08 & 0.881 & 27.08 & 0.927 & 29.91 \\
$[22.02, 22.40]$ & 0.911 & 28.20 & 0.882 & 27.17 & 0.928 & 30.04 \\
$[22.40, 22.77]$ & 0.911 & 28.14 & 0.884 & 27.21 & 0.927 & 30.17 \\
$[22.77, 23.14]$ & 0.910 & 28.17 & 0.884 & 27.33 & 0.927 & 30.20 \\
$[23.14, 23.64]$ & 0.911 & 28.28 & 0.883 & 27.36 & 0.926 & 30.35 \\
$[23.64, 24.27]$ & 0.910 & 28.25 & 0.883 & 27.27 & 0.928 & 30.53 \\
$[24.27, 26.73]$ & 0.910 & 28.37 & 0.882 & 27.42 & 0.928 & 30.87 \\
\bottomrule
\end{tabular}
\end{table*}

\Cref{synthetic_width,synthetic_coverage}
 compare our methods with CP-MDA methods missingness-wise. For clarity, 
 we applied a window smoothing with size of 5 to the curves. The unsmoothed curves are provided in Appendix~\ref{appendix:preimpute},
 and both lead to the same conclusion.
 These results confirm our theoretical MCV guarantee as established in \Cref{thm:weighted_cp_mcv,thm:ar_correction}, for all missingness patterns. 
 Our methods significantly reduce prediction interval width. In contrast, MDA-Exact fails to generate informative prediction sets for certain missingness patterns, 
  while MDA-Nested is more conservative. Overall, our methods achieve an average \textbf{30\% reduction}\footnote{The average of interval width over all 1023 missingness patterns is computed to compare the interval width.} in prediction width compared to MDA-Nested.

We also generate MAR and MNAR~(self-masked) mechanisms on the same synthetic dataset by methods in \citet{schouten2018generating, zaffran2024predictiveuncertaintyquantificationmissing}. The first 5 covariates $\{X_1, X_2, X_3, X_4, X_5\}$ may be missing, with total missingness probability 20\%. Under MAR, $M$ is a function $h_{\text{MAR}}(X_6,...,X_{10})$ and under MNAR, it is $h_{\text{MNAR}}(X_1,...,X_{5})$. 
We report the 95\% confidence interval for the worst-case coverage for 31 missingness patterns in \Cref{tbl:mar_mnar_results}. One should be careful when generating points for each pattern as $(X, Y)|M=m$ does not have an explicit form under MAR or MNAR. We repeat 100 experiments.
Our methods are MCV under both MAR and MNAR, with ARC CP providing the shortest prediction intervals.

\paragraph{Naive imputation baselines.}
We also evaluate two naive strategies on the synthetic dataset.
First, 
we impute all covariates and apply standard split CP on the fully imputed data.
This achieves marginal coverage of $90.6\%$ but violates mask-conditional validity: among the $2^{10}=1024$ masks, more than 100 have mask-conditional coverage below $90\%$, and the worst-case coverage drops to $61.8\%$.
Second, we train a quantile regression model (HistGradientBoostingRegressor with pinball loss) and compute prediction intervals directly.
While this yields relatively short average intervals (width $\approx 3.17$), the worst-case mask-conditional coverage is only $71.2\%$.
These baselines show that naive imputation-based approaches can achieve marginal average, but fail to adapt to heterogeneity across missingness patterns, and under-cover for some masks.

\subsection{High-dimensional data ($d=100$).}
We further consider a 100-dimensional version of the same dataset in Section~4.1. The exact $\beta$ is detailed in Appendix~\ref{appendix:100beta}.
Since there are $2^{100}$ possible masks, we only sample 1000 masks uniformly at random and, for each mask, generate
100 test points to estimate its mask-conditional coverage.

For each mask $m$, we compute the oracle $90\%$ interval width
$L_{\text{oracle}}(m)$, i.e., the shortest possible interval achieving $90\%$
coverage under the true data-generating model. We consider $L_{\text{oracle}}(m)$ as
a difficulty measure and group masks into $10$ bins according to this value. 
In each bin, we report the average mask-conditional coverage and
average interval width; see Table~\ref{tbl:100d_oracle}.

Across all oracle-difficulty bins, ARC-CP attains coverage around $91\%$ with
average interval widths between $27.9$ and $28.4$, and never produces infinite
intervals, whereas MDA-Exact always fails due to limited sample. Overall, ARC-CP offers a
favorable trade-off between coverage, efficiency, and runtime in this
100-dimensional regime.

\subsection{Experiments on real dataset}
\textbf{Concrete Data from the UCI repository} (\citet{concrete_compressive_strength_165}, \textbf{licensed under CC BY 4.0}): The goal is to predict compressive strength using 8 features. 630 training data and 100 calibration data are selected. For MCAR, 100 test data are generated per mask, with 50\% feature-wise missingness, resulting in 255 patterns ($2^8 - 1$). All four methods on all missingness patterns achieve 90\% coverage\footnote{See \Cref{appendix:comp}, \Cref{img:concrete_coverage} for details about missingness-wise coverage results}. Compared with MDA-Nested, ARC CP reduced the average width \textbf{by
12\%} and Weighted CP reduced average width \textbf{by 10\%}. 

For MAR, the first 4 features may be missing, with rates depending on the other 4 (totaling 20\%). Our methods are MCV with worst-case coverages: MDA-Exact $89.1\% \pm 1.64\%$, MDA-Nested $90.8\% \pm 1.43\%$, ARC CP $90.1\% \pm 1.63\%$, and  Weighted CP $90.6\% \pm 1.42\%$. ARC CP yields the narrowest average intervals—reducing width \textbf{by 7.5\%} relative to MDA-Nested (see \Cref{tbl:marconcrete}).
\textbf{Three additional real datasets} with higher dimensions are provided in  Appendix~\ref{bike_sharing}, which show the valid coverage of our methods.

\subsection{Ablation study of weight correction}
\label{section:weight correction}
To demonstrate the effectiveness of correction methods, we did experiments on the same synthetic data without correction, i.e., we pre-impute the calibration and apply the mask of the test point, ignoring the distribution shift. The results are summarized in \Cref{img:correction_app,img:correction_width_app} in Appendix~\ref{appendix:correction}, showing that the three methods have very similar interval widths. But without correction, the worst mask conditional coverage is around 89\%, indicating that without a suitable weight/acceptance for correction, the distribution shift results in an under-coverage. This confirms the necessity of correction when applying CP under missing data.

\subsection{Impact of imperfect estimators} 
Accurately estimating the likelihood ratio is never trivial. We conduct a synthetic case study (see Appendix~\ref{appendix:estimator quality}) to explicitly evaluate the effect of the likelihood ratio estimator’s quality on the coverage. Using a Gaussian generative model where the true density ratio is known, we show that (i) inaccurate estimators can lead to under- or over-coverage, and (ii) the empirical coverage correlates strongly with the estimator quality (e.g., Pearson correlation with the true ratio). Notably, we observe that when the correlation exceeds 0.3, the coverage  approaches the target $1 - \alpha$. These results suggest that targeting exact density ratio values might be unnecessary—what matters is preserving the relative ordering of likelihoods. 
This tolerance to estimation noise makes the method practically robust.

\section{Conclusion and future works}
In this work, we introduce a new framework called preimpute–mask–then–correct framework for conformal prediction with missing data, together with two novel correction strategies. 
Our two correction methods, Mask-Conditional Weighted CP and ARC CP, are developed to overcome the key limitations of existing methods, 
specifically, the conservativeness of CP-MDA-Nested and the impractical sample requirements of CP-MDA-Exact, by leveraging generalized distributional imputation and weighted CP.
Our theoretical analysis establishes that both methods not only guarantee marginal coverage but also mask conditional coverage under general missingness mechanisms including MCAR, MAR, and MNAR with suitable conditions. 
Extensive experiments on synthetic and real-world datasets demonstrate the effectiveness of our methods. Compared to prior approaches, our methods significantly reduce the width of prediction set while preserving the desired validity properties. Notably, our ARC method achieves a competitive balance between computational efficiency and predictive informativeness. 
It reduces 90\% of the computing time compared to traditional weighted CP in the case of distribution shift (see Appendix~\ref{app:running time}), making it a promising substitution candidate.
Future work includes deriving tighter robustness bounds under likelihood ratio estimation error or violations of \Cref{assump:absolute_continuous}, and developing adaptive strategies for high-dimensional or imperfect imputations. 

\subsubsection*{Acknowledgements}
We thank the anonymous reviewers for their insightful comments and suggestions, which have helped improve the paper.
This project has received funding from the European Union’s Horizon Europe research and innovation programme under the Marie Skłodowska-Curie COFUND grant agreement No 101127936 (DeMythif.AI). This research was also supported by DATAIA Convergence Institute as part of the « Programme d’Investissement d’Avenir », (ANR-17-CONV-0003) operated by Universit\'e Paris-Saclay.

\bibliography{references}

\section*{Checklist}



\begin{enumerate}

  \item For all models and algorithms presented, check if you include:
  \begin{enumerate}
    \item A clear description of the mathematical setting, assumptions, algorithm, and/or model. [Yes]
    \item An analysis of the properties and complexity (time, space, sample size) of any algorithm. [Yes] see \Cref{tbl:runningtime}.
    \item (Optional) Anonymized source code, with specification of all dependencies, including external libraries. [Yes] We will upload it along with the supplementary material.
  \end{enumerate}

  \item For any theoretical claim, check if you include:
  \begin{enumerate}
    \item Statements of the full set of assumptions of all theoretical results. [Yes]
    \item Complete proofs of all theoretical results. [Yes]
    \item Clear explanations of any assumptions. [Yes]     
  \end{enumerate}

  \item For all figures and tables that present empirical results, check if you include:
  \begin{enumerate}
    \item The code, data, and instructions needed to reproduce the main experimental results (either in the supplemental material or as a URL). [Yes]
    \item All the training details (e.g., data splits, hyperparameters, how they were chosen). [Yes]
    \item A clear definition of the specific measure or statistics and error bars (e.g., with respect to the random seed after running experiments multiple times). [Yes]
    \item A description of the computing infrastructure used. (e.g., type of GPUs, internal cluster, or cloud provider). [Yes]
  \end{enumerate}

  \item If you are using existing assets (e.g., code, data, models) or curating/releasing new assets, check if you include:
  \begin{enumerate}
    \item Citations of the creator If your work uses existing assets. [Yes]
    \item The license information of the assets, if applicable. [Yes]
    \item New assets either in the supplemental material or as a URL, if applicable. [Yes]
    \item Information about consent from data providers/curators. [Not Applicable]
    \item Discussion of sensible content if applicable, e.g., personally identifiable information or offensive content. [Not Applicable]
  \end{enumerate}

  \item If you used crowdsourcing or conducted research with human subjects, check if you include:
  \begin{enumerate}
    \item The full text of instructions given to participants and screenshots. [Not Applicable]
    \item Descriptions of potential participant risks, with links to Institutional Review Board (IRB) approvals if applicable. [Not Applicable]
    \item The estimated hourly wage paid to participants and the total amount spent on participant compensation. [Not Applicable]
  \end{enumerate}

\end{enumerate}

\clearpage
\appendix

\thispagestyle{empty}

\onecolumn
\aistatstitle{Supplementary Materials}

\input{supplement}
\end{document}

%% file: supplement.tex
\section{Discussion about quality of pre-imputation for calibration dataset}
\label{appendix:preimpute}
In our work, only the calibration dataset requires imputation. For the basemodel $\hat{f}:\mathcal{X}_{\mathrm{NA}} \rightarrow \mathcal{Y}$ that is trained on the training data, it could contain implicitly an imputation  (e.g. $\hat{f}$ could be a composition of a neural network $f_{\mathrm{NN}}:\mathcal{X} \rightarrow \mathcal{Y}$ and an imputation function $\varphi$, such that $\hat{f}=f_{\mathrm{NN}}\circ \varphi$) or it could be simple a tree model without imputation. Therefore, our method can be intergrated for any basemodel $\hat{f}$.

However, a key step for our method is the pre-imputation on calibration dataset. This section discusses the significance of pre-imputation quality for calibration datasets in uncertainty quantification. We begin by reviewing advanced imputation methods for missing values, followed by a rationale for adopting distributional imputations. Finally, we compare the performance of deterministic and distributional imputation methods through experiments.

Imputation methods can be broadly categorized into three types \citep{jarrett2022hyperimpute}: (1) \textbf{Simple value imputation}: Methods such as mean or median imputation use representative values for missing data without reconstructing the original data distribution.(2) \textbf{Iterative imputation} \citep{liu2014stationary,zhu2015convergence}:, such as MICE \citep{rubin1996multiple}, Missforest \citep{stekhoven2012missforest}, etc. In these approaches, the conditional distributions of each feature are estimated based on all the others, and missing values are imputed using such univariate models in a round-robin fashion until convergence. More recently, \citet{naf2024good} proposed mice-Distributional Random Forest (DRF) by leveraging the ability to approximate the conditional distribution of DRF and argued that for MAR condition, it is important to have a distributional imputation function rather than deterministic imputation, one of the three criteria they proposed for being a good imputation candidate. All these methods require a pre-defined functional form for each variable. (3) \textbf{Deep generative models}: such as GAN-based imputations \citep{yoon2018gain,yoon2020gamin} or VAE-based imputations \citep{du2023remasker,nazabal2020handling, ma2020vaem, mattei2019miwae,fortuin2020gp}. These approaches may leverage better capacity and efficiency of learning using deep function approximators and the ability to capture correlations among covariates by amortizing the parameters \citep{jarrett2022hyperimpute}. While powerful, these methods typically require large datasets and may be computationally expensive.

However, when it comes to imputation in an $X-Y$ type dataset with missing values in $X$, the traditional imputation is derived only as a function of $X_{\text{obs}(M)}$ predicting the missing part $X_{\text{mis}(M)}$, which limits the potential of imputation in a more general scenario. Under such a definition and for {\it impute-then-regression} tasks, \citet{le2021sa} has shown that ``imputation doesn't matter'' in the sense that a universal learner can achieve the Bayes optimal for almost all imputation functions. 
Yet, for uncertainty quantification and coverage guarantees, this is not sufficient, so we propose a broader definition of imputation, extending to distributional and generalized distributional imputations. We will demonstrate that a good imputer can help achieve the target coverage for the prediction interval if it employs generalized distributional imputation.

In the main text, we do not give a formal definition of the imputation function; instead, we use only $(\hat{X},Y)$ to represent imputed points. In this part, we formalize three types of imputation functions as follows and use $\varphi$ to represent an imputation function such that $\hat{X}\stackrel{d}{=}\varphi(\widetilde{X})$ or $\hat{X}\stackrel{d}{=}\varphi(\widetilde{X}, Y)$, depending on whether $\varphi$ takes $Y$ as input or not.
\begin{definition}[Deterministic imputation function (Dt)]
    We say $\varphi_{Dt}:\mathcal{X}_{\mathrm{NA}} \rightarrow \mathcal{X}$ is a deterministic imputation function, if it satisfies $\forall x\in \mathcal{X},m \in \{0,1\}^d$
    $$\varphi_{Dt}(\text{mask}(x,m))_{\text{obs}(m)}=x_{\text{obs}(m)}.$$
\end{definition}

\begin{definition}[Distributional imputation function (Ds)]
    We say a function $\varphi_{Ds}: \mathcal{X}_{\mathrm{NA}}  \to \mathcal{Z}(\mathcal{X})$
where $\mathcal{Z}(\mathcal{X})$ denotes the space of random variables over $\mathcal{X}$ is a distributional imputation function if it satisfies  $\forall x\in \mathcal{X},m \in \{0,1\}^d$:
    \[
P(\varphi_{Ds}(\text{mask}(x,m))_{\text{obs}(m)} = x_{\text{obs}(m)}) = 1.
\]
The missing part of the covariate, $X_{\text{mis}(M)}$, is then imputed by $\varphi_{Ds}(X)_{\text{mis}(M)}$. 

\end{definition}

The distributional imputation function can be seen as the first step of multiple imputations (MI) \citep{rubin1996multiple}, which fuses multiple estimations. 

\begin{definition}[Generalized distributional imputation function (GDs)]
    We say a function $\varphi_{GDs}: \mathcal{X}_{\mathrm{NA}} \times  \mathcal{Y} \to \mathcal{Z}(\mathcal{X}),$
 is a generalized distributional imputation function if it satisfies  $\forall x\in \mathcal{X},m \in \{0,1\}^d,y\in\mathcal{Y}$:
   \[
P(\varphi_{GDs}(\text{mask}(x,m),y)_{\text{obs}(m)} = x_{\text{obs}(m)}) = 1.
\]
The missing part of the covariate, $X_{\text{mis}(M)}$, is then imputed by $\varphi_{GDs}(\widetilde{X}(M),Y)_{\text{mis}(M)}$.
\end{definition}

Here, we present an additional experiment to compare the effectiveness of deterministic and distributional imputations on synthetic datasets (\cref{Section:exp on synthetic mcar}). To simulate  \textbf{deterministic imputation}, we disable posterior sampling in iterative Bayesian Ridge imputation, resulting in fixed imputed values.

\begin{figure}[H]
\begin{center}
\centerline{\includegraphics[width=0.8\textwidth]{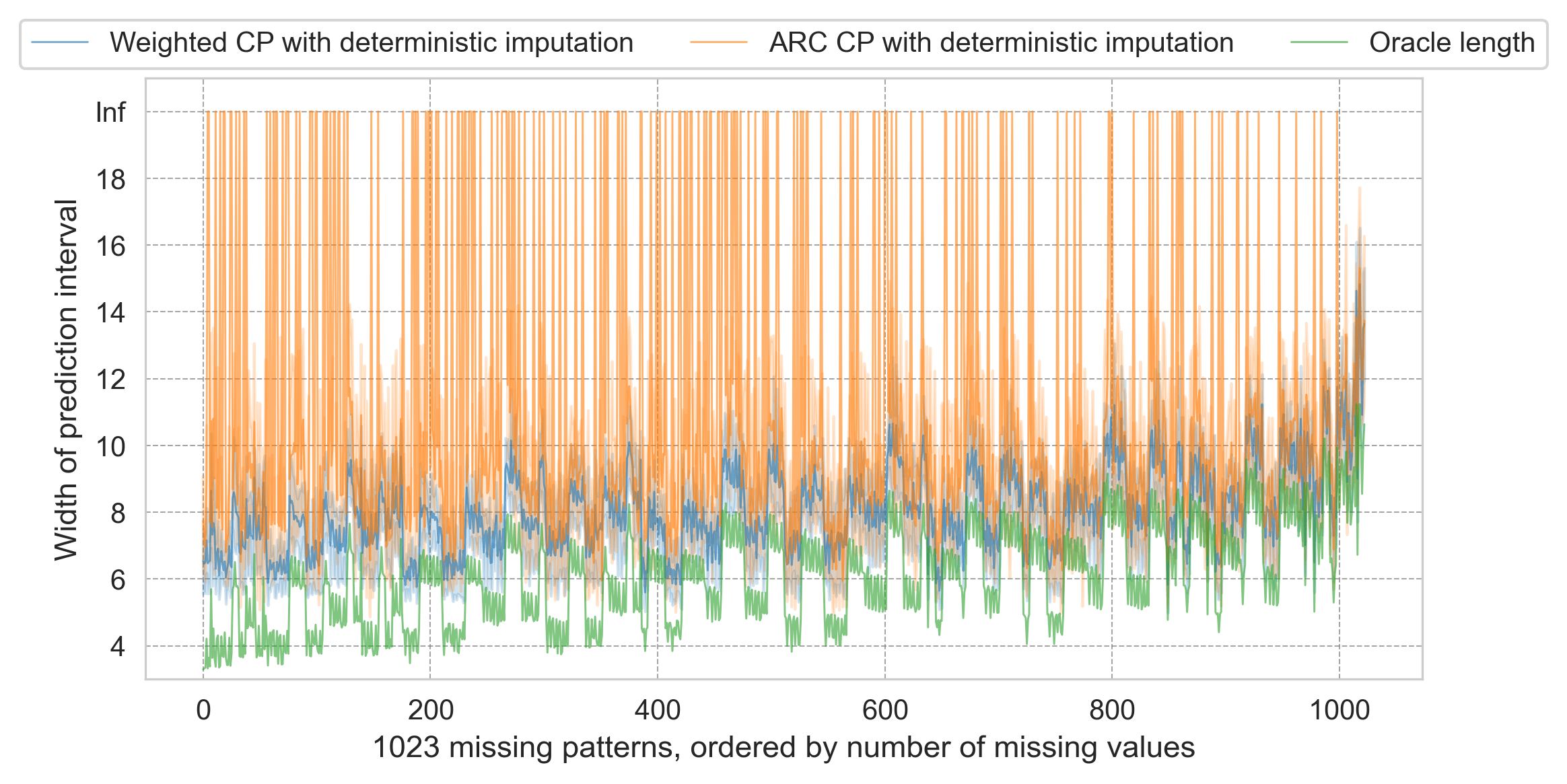}}
\caption{Synthetic dataset: mean and 99\% confidence interval of prediction interval width with respect to 1023 missingness patterns, results by using deterministic imputation.}
\label{img:width_deterministic}
\end{center}
\vskip -0.2in
\end{figure}

\begin{figure}[H]
\begin{center}
\centerline{\includegraphics[width=0.8\textwidth]{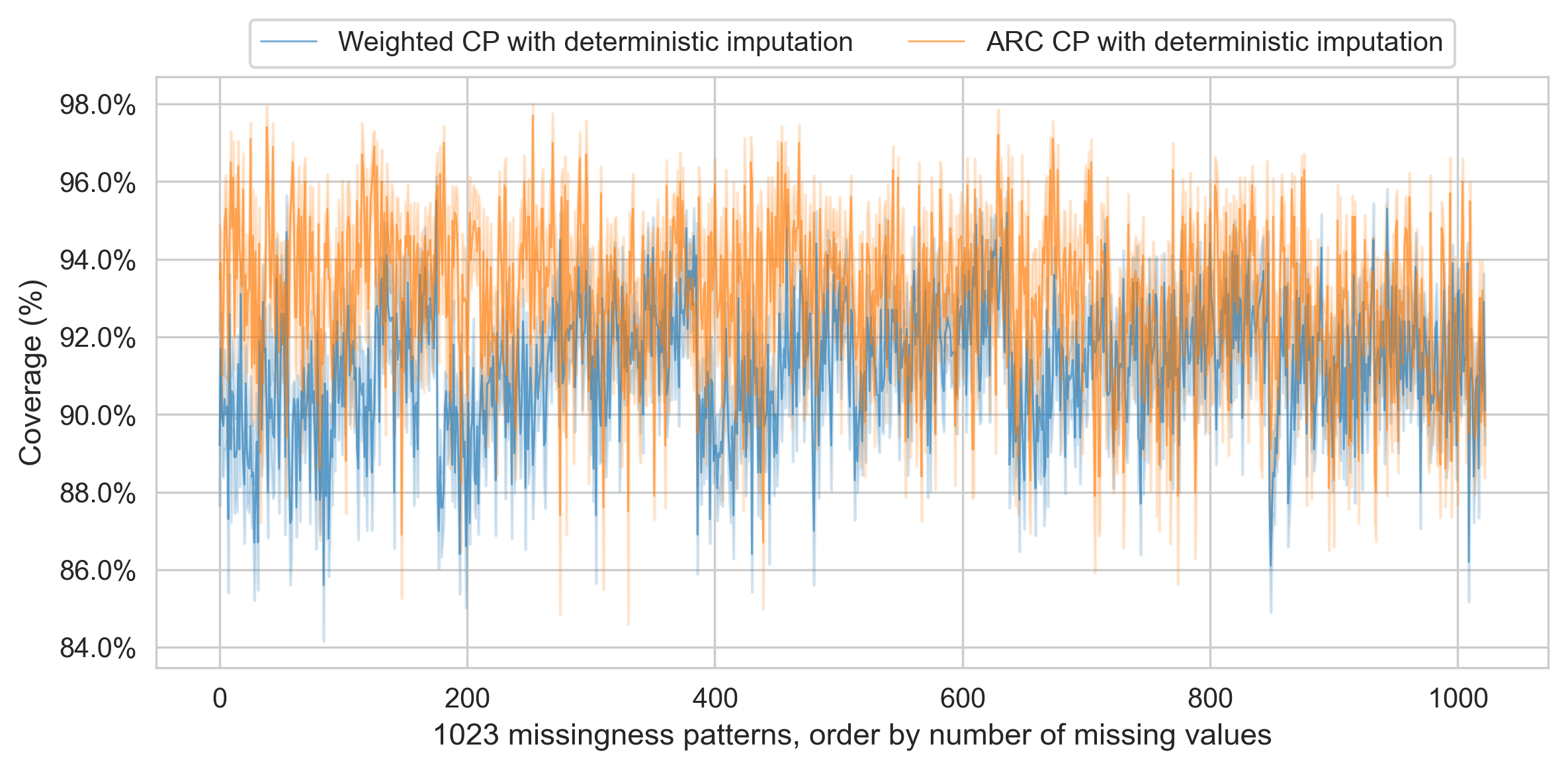}}
\caption{Synthetic dataset: mean and 99\% confidence interval of prediction coverage with respect to 1023 missingness patterns, results by using deterministic imputation.}
\label{img:coverage_deterministic}
\end{center}
\vskip -0.2in
\end{figure}
Compared with \Cref{synthetic_width} and \Cref{synthetic_coverage}, we observe from \Cref{img:width_deterministic,img:coverage_deterministic} that if we use a deterministic imputation instead of a distributional imputation, and keep everything else the same, the results may be suboptimal. This outcome aligns with our expectations: if we impute missing values using a deterministic function, the resulting imputed distribution would lack continuity and violate the absolute continuity assumption in \Cref{assump:absolute_continuous}, leading to two potential drawbacks:

\begin{enumerate}
    \item For Weighted CP: if the likelihood ratio $\omega_m$ of a specific point is a high value, say $\omega_m(x^j_m,y^j) \approx \infty$, it is highly likely that $p^j_m=1$ from definition of \Cref{thm:weighted_cp_mcv}. Consequently, when computing the weighted quantile as in \Cref{eq:weighted_prediction_set}, the weighted quantile will equal $s^j_m$. In other words, if even one likelihood ratio is extremely large, the weight can easily be influenced by that extreme value, distorting the weighted quantile. This distortion may result in inaccurate quantile computation, which can, in many practical scenarios, fail to guarantee MCV.
    \item For ARC CP: if the absolute continuity condition is not satisfied, implying that the likelihood ratio is unbounded (i.e., $K \to \infty$), most points are likely to be rejected during the acceptance-rejection step. This would lead to an insufficient number of samples in the calibration dataset, potentially resulting in infinitely wide prediction intervals for certain missingness patterns
\end{enumerate}

These results highlight that the quality of imputation plays a crucial role in our framework.
If the imputation is poor or purely deterministic, the resulting method degenerates, both theoretically and practically, toward CP-MDA-Exact, because the imputed calibration data no longer provides additional distributional information beyond the observed part.
In contrast, a well-calibrated distributional imputation allows our method to exploit richer information from the calibration data, effectively correcting the mask-induced distribution shift and achieving tighter prediction intervals with valid coverage.
Hence, the pre-imputation step is not only a preprocessing choice, but a key component that governs the efficiency and robustness of our uncertainty quantification framework.

\subsection{Sufficient conditions for absolute continuity}
\label{app:absolute_continuity}

In this section we provide a simple sufficient condition ensuring the absolute
continuity assumption in Assumption~\ref{assump:absolute_continuous} holds under
distributional imputation.

\begin{proposition}[Sufficient conditions for absolute continuity under distributional imputation]
\label{prop:absolute_continuity}
Let $P_m$ be the true distribution of $(X_{\obs(m)},Y)\mid M=m$, and let $Q_m$
be the distribution of the imputed pair $(\hat X_{\obs(m)},Y)$ as defined in
Section~3.1. Assume that $P_m$ and $Q_m$ admit densities $p_m$ and $q_m$ with
respect to a common dominating measure $\mu$ on
$\mathcal{X}_{\obs(m)} \times \mathcal{Y}$ (Lebesgue for continuous components,
counting measure for discrete components), i.e.\ $P_m = p_m \cdot \mu$ and
$Q_m = q_m \cdot \mu$. Then $P_m$ is absolutely continuous with respect to
$Q_m$ if and only if
\[
\forall (x_{\obs},y)\ \text{with } p_m(x_{\obs},y)>0,\quad
q_m(x_{\obs},y)>0.
\tag{$\star$}
\]
Condition~$(\star)$ implies that the support of $P_m$ is contained in the
support of $Q_m$. In particular, $(\star)$ is satisfied in each of the
following cases:
\begin{itemize}
  \item \textbf{Continuous case.} $P_m$ has a density supported on a subset of
  $\mathcal{X}_{\obs(m)} \times \mathcal{Y}$ without point masses, and the
  imputation mechanism induces a density $q_m$ that is strictly positive on the
  same support (e.g.\ Gaussian or Bayesian regression–based MICE, which has
  continuous support on the whole domain).

  \item \textbf{Discrete case.} $P_m$ is discrete and supported on a finite or
  countable set of atoms. Then $(\star)$ holds as soon as the imputation
  distribution assigns non-zero probability to every such atom.

  \item \textbf{Mixed case.} $P_m$ has both a continuous part and atoms, and
  the imputation distribution has a positive density on the continuous support
  and assigns non-zero mass to all atoms.
\end{itemize}
In all three cases the likelihood ratio
\[
\omega_m(z) = \frac{dP_m}{dQ_m}(z)
\]
is well defined for all $z$ in the support of $P_m$.
\end{proposition}

The assumption $P_m \ll Q_m$ is therefore related to the choice of distributional imputer. For MAR/MNAR cases, the missingness does not directly break absolute continuity between $P_m$ and $Q_m$. In practice, MAR/MNAR make the conditional law harder to estimate, which can increase the likelihood-ratio $\omega_m$ and thus the divergence between $P_m$ and $Q_m$. However, when all the covariates are continuous over $\mathcal{X}$, most of the modern distributional imputation measures lead to  $P_m \ll Q_m$, typically, this is the case when $Q_m$ is Gaussian with MICE.

Moreover, as we emphasized at the end of Appendix A (Figures 4 and 5), deterministic imputers (e.g., mean imputation) may violate this condition and cause our methods to fail, whereas simple distributional imputers, such as MICE, yield stable results.

\section{Implementation details of Weighted CP}
\label{appendix:weighted_cp_implementation}

As discussed in the main text, weighted CP in its general form involves the test label $y$ during inference, which poses challenges for efficient implementation. Nevertheless, as shown in Table~\ref{tbl:runningtime}, the weighted CP method can still be implemented with success in practice. We have proposed a series of optimizations that significantly reduce the computational burden, making it a viable choice. The optimizations are as follows:

\subsection{Search for Interval Bounds}

It is often reasonable to assume that the prediction set defined by \Cref{eq:weighted_prediction_set} is an interval $[B_l, B_u]$, otherwise, the prediction set is included in such an interval. This assumption allows us to reduce the computational cost by searching only for the lower and upper bounds of the prediction interval, instead of performing a full grid search across $\mathcal{Y}$.

\subsection{Coarse Grid Search}

To further accelerate the search, we introduce a step size $h$ and define the search grid as:
\[
\mathcal{G} = \{ y_{\min}, y_{\min} + h, y_{\min} + 2h, \dots, y_{\max} \}.
\]
We then define the estimated bounds as:
\[
\hat{B}_l = \min \left\{ y \in \mathcal{G} \mid y \notin \widehat{C}_{\alpha}(x),\ y + h \in \widehat{C}_{\alpha}(x) \right\},
\]
\[
\hat{B}_h = \max \left\{ y \in \mathcal{G} \mid y \in \widehat{C}_{\alpha}(x),\ y + h \notin \widehat{C}_{\alpha}(x) \right\}.
\]
This coarse discretization ensures that the conservativeness of the resulting interval is at most $2h$, while validity is preserved.

\subsection{Predefined Search Range}

We restrict the search range to a data-dependent interval to avoid unnecessary computation. Specifically, we define:
\[
y_{\min} = \min \left\{ y \mid s(\tilde{x}, y) \leq \max_{i=1,\dots,n} S^i_m \right\}, \quad
y_{\max} = \max \left\{ y \mid s(\tilde{x}, y) \leq \max_{i=1,\dots,n} S^i_m \right\}.
\]
Then the search is performed within $[y_{\min}, y_{\max}]$. To preserve theoretical validity, we use the following rules:
\[
\hat{B}_l = -\infty \quad \text{if } y_{\min} \in \widehat{C}_{\alpha}^W(\tilde{x}), \quad
\hat{B}_h = \infty \quad \text{if } y_{\max} \in \widehat{C}_{\alpha}^W(\tilde{x}).
\]

All the above techniques can make the Weighed CP more efficient and valid in theory, though being conservative at $2h$. While the AR-Corrected method is our primary recommendation for practical deployment due to its speed and flexibility, the weighted CP approach remains valuable. In particular, weighted CP avoids the added variance introduced by sub-sampling in the AR correction procedure, making it more stable in scenarios where variance matters.

\section{Proof}
\label{appendix:proof}
Assume that $P$, $Q$, and for any $m \in \{0,1\}^d$ $P_m$ and $Q_m$ admit densities with respect to common dominating measure $\mu$. We denote these densities respectively 
$p$, $q$, $p_m$, and $q_m$. 

\subsection{Proof of \Cref{prop:bound}}
\label{proof:prop:bound}
Recall the definition of $Q_{M^{n+1}}$ and $P_{M^{n+1}}$ which are, the distribution of imputed calibration with missingness $M^{n+1}$ and the distribution of test point, conditionally on $M^{n+1}$. By making a direct Split CP on the two distribution conditionally on  $M^{n+1}=m$, by Theorem 2 of \citep{barber2023conformal}, we have for all $m \in \{0,1\}^d$ such that $P(M=m)>0$:
\begin{equation}
    P(Y^{n+1}\in \hat{C}_{\alpha}(\widetilde{X}^{n+1})\mid M^{n+1}=m) \geq 1-\alpha-2d_{TV}(Q_{m},P_{m})
\end{equation}

For the term of distance of total variation, we have:
\begin{align*}
    &2d_{TV}(Q_{m},P_{m})\\
    &=2\times \frac{1}{2}\int_{\mathcal{X}^{\text{obs}(m)}\times \mathcal{Y}} \big|p_{m}(x_{\text{obs}(m)},y)-q_{m}(x_{\text{obs}(m)},y)\big|d(x_{\text{obs}(m)},y) \\
    &=\int_{\mathcal{X}^{\text{obs}(m)}\times \mathcal{Y}} \bigg|
    \int_{\mathcal{X}^{\text{mis}(m)}}p(x,y| M=m)dx_{\text{mis}(m)}-\int_{\mathcal{X}^{\text{mis}(m)}}q(x,y)dx_{\text{mis}(m)}\bigg|d(x_{\text{obs}(m)},y) \\
    &=\int_{\mathcal{X}^{\text{obs}(m)}\times \mathcal{Y}} \bigg|
    \int_{\mathcal{X}^{\text{mis}(m)}}\big\{p(x,y|M=m)-p(x,y)+p(x,y)-q(x,y)\big\}dx_{\text{mis}(m)}\bigg|d(x_{\text{obs}(m)},y) \\
    &\leq \int_{\mathcal{X}^{\text{obs}(m)}\times \mathcal{Y}} 
    \int_{\mathcal{X}^{\text{mis}(m)}}\big|p(x,y| M=m)-p(x,y)+q(x,y)-q(x,y)\big|dx_{\text{mis}(m)} d(x_{\text{obs}(m)},y) \\
&\leq \int_{\mathcal{X}^{\text{obs}(m)}\times \mathcal{Y}} 
    \int_{\mathcal{X}^{\text{mis}(m)}}\big\{|p(x,y| M=m)-p(x,y)|+|p(x,y)-q(x,y)|\big\}dx_{\text{mis}(m)} d(x_{\text{obs}(m)},y) \\
&= \int_{\mathcal{X} \times \mathcal{Y}}\big\{|p(x,y| M=m)-p(x,y)|+|p(x,y)-q(x,y)|\big\}d(x,y) \\
&=
d_{TV}{(P,Q)}+\int_{\mathcal{X} \times \mathcal{Y}}\bigg|\frac{dP(x,y\mid M=m)}{dP(x,y)}-1 \bigg| dP(x,y)d(x,y)\\
    &= d_{TV}{(P,Q)}+ \mathbb{E}_{P}\left[\bigg|\frac{d{P}(X,Y|M=m)}{dP(X,Y)}-1\bigg|\right] \label{eq_last expectation}
\end{align*}
For the MCAR case, note that $p(X, Y|M)=p(X, Y)$; therefore, the second term in the last equation is zero.


\subsection{Proof of \Cref{thm:weighted_cp_mcv}}
\label{proof:thm:weighted_cp_mcv}


For any $m \in \{0,1\}^d$, let's consider all event conditionally on $M^{n+1}=m$: by definition of $Q_m$ and that $M^{n+1} \perp \{\hat{X}^i_{\text{obs}(M^{n+1})},Y^i)\}_{i=1}^{n}$, we have the imputed-then-masked calibration points $(\{\hat{X}^i_{\text{obs}(M^{n+1})},Y^i)\}_{i=1}^{n}\mid M^{n+1}=m) \sim Q_m$ and the test point $(({X}^{n+1}_{\obs(M^{n+1})},Y^{n+1})\mid M^{n+1}=m) \sim P_{m} $.

Recall that $P_m$ and $Q_m$ are probability defined on $\# \obs(m)$-dimensional space.  Now consider $d-$dimensional dataset $(\{\mask (\hat{X}^i,M^{n+1}),Y^i\}_{i=1}^n\mid M^{n+1}=m)$, whose distribution is denoted by $Q^\star_m$ and the test point $(\widetilde{X}^{n+1},Y^{n+1}) \mid M^{n+1}=m)$ with its distribution denoted by $P_m^*$\footnote{Or it is equivalent to the distribution $\mask (X^{n+1},M^{n+1}) \mid M^{n+1}=m)$, since $\widetilde{X}^{n+1}=\mask (X^{n+1},M^{n+1})$}. 
Suppose $Q^\star_m$ and $P^\star_m$ are dominated by the measure $\mu$ with densities $q^\star_m$ and $p^\star_m$, respectively. By definition of the mask operator, one can easily check that:
\[
p^\star_m(x,y)=p_m(x_{\obs(m)},y)  \ind_{{x_\mis(m)}=NA}
\]
\[
q^\star_m(x,y)=q_m(x_{\obs(m)},y)  \ind_{{x_\mis(m)}=NA}
\]

By applying Lemma 2 in \citep{tibshirani2019conformal} and Assumption~\ref{assump:absolute_continuous}, we know that all the $n+1$ points are weighted exchangeable with weight functions $\omega_i=1$ for $i=1,...,n$ and $\omega_{n+1}=\frac{p^{\star}_m}{q^{\star}_m}=\frac{p_m}{q_m}=\frac{dP_m}{dQ_m}$.

By applying Theorem 2 in \citep{tibshirani2019conformal}, we deduce immediately that, conditionally on $M^{n+1}=m$, the prediction set $\widehat{C}^{W}_{\alpha}(\widetilde{X}^{n+1})$ is exactly the weighted version and thus  \[P(Y^{n+1} \in \widehat{C}^{W}_{\alpha}(\tilde{X}^{n+1})\mid M^{n+1}=m) \geq 1-\alpha.\]


The formal algorithm is summarized in \Cref{alg:weighted_cif}:

\begin{algorithm}[hbt!]
   \caption{Weighted CP with pre-imputation of Calibration}
   \label{alg:weighted_cif}
   \textbf{Input:} Calibration set $\mathrm{Cal}=\{(\widetilde{X}^{k}, Y^{k})\}_{k=1}^n$, generalized distributional imputation $\varphi$, predictive model $\hat{f}$, nonconformity score function $s_{\hat{f}}(\tilde{x},y)$, likelihood estimation function $\hat{\omega}_m(\mask(x,m),y)$ for $\omega_m(x_{\text{obs}(m)},y)$, significance level $\alpha$, test point $\widetilde{X}^{n+1}$. Imputation function $\varphi$ and the predictive model $\hat{f}$ are trained on an independent training set. \\
   \textbf{Output:} Prediction set $\widehat{C}^W_{\alpha}(\widetilde{X}^{n+1})$.

   \begin{algorithmic}[1]
   
   \FOR{$k \in \{1,...,n\}$}
       \STATE $(\widetilde{X}^{k}, Y^{k}) = (\varphi(\widetilde{X}^{k}, Y^{k}), Y^{k})$ 
       \textcolor{blue}{// Pre-impute Cal}
       \STATE  $(\hat{X}^{k}, Y^{k}) = (\mask(X^{k},{ M^{n+1})}, Y^{k})$
       \textcolor{blue}{// Add missingness $M^{n+1}$ on each data point}
       \STATE $\omega^{k} = \hat{\omega}_m(\mask(\hat{X}^{k},M^{n+1}),  Y^{k})$ \textcolor{blue}{//Compute weight}
   \ENDFOR
   \STATE Define $W^i(x_{\mathrm{obs}(M^{n+1})},y)$ and $W^{n+1}(x_{\mathrm{obs}(M^{n+1})},y)$ as in \Cref{thm:weighted_cp_mcv}.
   \STATE Compute the weighted version of the prediction set as:
   \begin{align}
       &\widehat{C}^{W}_{\alpha}(\widetilde{X}^{n+1}) :=
   \Big\{ y : s_{\hat{f}}(\widetilde{X}^{n+1}, y) \nonumber \\
   &\leq \mathcal{Q}_{1-\alpha}\bigg(\sum_{i=1}^{n} W^{i}(X^{i}_{\mathrm{obs}(M^{n+1})}, y) \delta_{s_{\hat{f}}(\widetilde{X}^i,  Y^{i})} + W^{n+1}(X^{n+1}_{\mathrm{obs}(M^{n+1})}, y) \delta_{\infty} \bigg) \Big\}. 
   \end{align}
   \end{algorithmic}
\end{algorithm}
\subsection{Proof of \Cref{thm:ar_correction}}
\label{proof:thm:ar_correction}
Before proceeding with the proof, we clarify the assumption made in \Cref{thm:ar_correction} that each calibration sample’s importance weight $\omega^i_m$ is bounded by a constant $K > 0$. In fact, this is equivalent to assuming that the likelihood ratio function $\omega_m(x_{\obs(m)}, y) := \frac{P_m(x_{\obs(m)}, y)}{Q_m(x_{\obs(m)}, y)}$ is uniformly bounded for all $(x_{\obs(m)}, y)$.

\begin{proof}

1. For the first point, since all data have missing value $\mathrm{NA}$ for the index mis($m$) of $X$, we only need to prove that 
all points have the same marginal distribution on the index obs($m$) of $X$ and $Y$, and they are independent.

To prove that they have the same distribution, it is sufficient to prove that 
the conditional distribution of $(\hat{X}^i_{\text{obs}(m)}, Y^i)$ given that $U \leq \frac{\omega^i_m}{K}$ is indeed the same distribution of $P_m$, That is, in terms of cumulative distribution function, we wish to show that
\begin{equation}
    P\bigg(\hat{X}_{\text{obs}(m)}\leq x, Y\leq y\mid U \leq \frac{\beta}{K}\frac{dP_m}{dQ_m}(\hat{X}_{\text{obs}(m)},Y)\bigg)=P(X_{\text{obs}(m)}\leq x, Y\leq y \mid M=m),
    \label{eq:target_ar2}
\end{equation}
where $\beta$ is an unknown normalization constant such that $\omega_m^i=\beta \frac{dP_m}{dQ_m}$. 
Letting $B=\{U \leq \frac{\beta}{K}\frac{d{P}_m}{dQ_m}(\hat{X}_{\text{obs}(m)},Y)\}$ be an event, $A=\{\hat{X}_{\text{obs}(m)}\leq x, Y\leq y\}$, we have firstly by the definition of uniform distribution:

$$P(B\mid {\hat{X}_{\text{obs}(m)}=x, Y=y})=\frac{\beta}{K}\frac{dP_m(x,y)}{dQ_m(x,y)},$$
We have also:
\begin{align*}
    P(B)&=\int_{\mathcal{X}_{\obs(m)}\times \mathcal{Y}}P(B\mid {\hat{X}_{\text{obs}(m)}=x, Y=y})dQ_m(\hat{X}_{\text{obs}(m)}=x,Y=y)\\
    &=\int_{\mathcal{X}_{\obs(m)}\times \mathcal{Y}}\frac{\beta}{K}\frac{dP_m(x,y)}{dQ_m(x,y)}dQ_m(x,y) \\
    &= \frac{\beta}{K},
\end{align*}
\begin{align}
    Q_m(B\mid A)
&= \frac{\int_A \mathbb{P}_{\mathrm{U\sim Uniform}}\!\left(U\le \tfrac{\beta}{K}\,\omega_m(z)\mid z\right)\,dQ_m(z)}{Q_m(A)} \\
&= \frac{\int_A \tfrac{\beta}{K}\,\omega_m(z)\,dQ_m(z)}{Q_m(A)} \\
&= \frac{\beta}{K\,Q_m(A)}\int_A dP_m \\
& = \frac{\beta\,P_m(A)}{K\,Q_m(A)}.
\end{align}
By Bayes's rule:
\begin{align*}
    P(B|A)&=\frac{P(B,A)}{P(A)}\\
    &=\frac{P(U \leq \frac{\beta}{K}\frac{d{P}_m}{dQ_m}(\hat{X}_{\text{obs}(m)},Y),\hat{X}_{\text{obs}(m)}\leq x, Y\leq y)}{P(\hat{X}_{\text{obs}(m)}\leq x, Y\leq y )}\\
    &=\int_{-\infty}^{(x,y)}\frac{P(U \leq \frac{\beta}{K}\frac{d{P}_m(w_x,w_y)}{dQ_m(w_x,w_y)}\mid w_x\leq x, w_y\leq y)}{P(\hat{X}_{\text{obs}(m)}\leq x, Y\leq y)} dQ_m(w_x,w_y)) \\
    &=\frac{1}{P(\hat{X}_{\text{obs}(m)}\leq x, Y\leq y\mid M=m)}\int_{-\infty}^{(x,y)}\frac{\beta}{K} \frac{d{P}_m(w_x,w_y)}{dQ_m(w_x,w_y)} dQ_m(w_x,w_y) \\
    &=\frac{\beta}{K\cdot P(\hat{X}_{\text{obs}(m)}\leq x, Y\leq y)}\int_{-\infty}^{(x,y)}d{P}_m(w_x,w_y) \\
    &= \frac{\beta \cdot P(X_{\text{obs}(m)}\leq x, Y\leq y\mid M=m)}{K\cdot P(\hat{X}_{\text{obs}(m)}\leq x, Y\leq y)}.
\end{align*}

Note that
$P(\hat{X}_{\text{obs}(m)}\leq x, Y\leq y\mid U \leq \frac{\beta}{K}\frac{dP_m}{dQ_m}(\hat{X}_{\text{obs}(m)},Y))=P(A|B)$, therefore we have:
\begin{align*}
    P(A|B)&=P(B|A)\frac{P(A)}{P(B)}\\
    &=\frac{\beta \cdot P(X_{\text{obs}(m)}\leq x, Y\leq y\mid M=m)}{K\cdot P(\hat{X}_{\text{obs}(m)}\leq x, Y\leq y)} \cdot \frac{P(\hat{X}_{\text{obs}(m)}\leq x, Y\leq y)}{\beta/K} \\
    &=P(X_{\text{obs}(m)}\leq x, Y \leq y \mid M=m),
\end{align*}
which is the target Equation~\eqref{eq:target_ar2}. By applying the acceptance-rejection method on the imputed calibration data, the obtained data points follow the same distribution as $(X_{\text{obs}(m)}, Y)$. 

Additionally, because the points in $ (\hat{X}^i, Y^i) \cup (X^{n+1}, Y^{n+1})$ are independent with each other by definition and $U \perp\!\!\!\perp (X^k, Y^k)_{k=1}^{n+1}$, all data points in $\widetilde{Cal}_m$ and $(X^{n+1}_{\text{obs}(m)}, Y^{n+1})$ are mutually independent, we conclude that all the data points are i.i.d.

2. As proved before, conditionally on $M=m$, all data points in $\widetilde{Cal}_m$ and $(X^{n+1}_{\text{obs}(m)}, Y^{n+1})$ are i.i.d., satisfying the condition (exchangeability) of Conformal Prediction, yielding for any marginal valid method with target $1-\alpha$:
$$P(Y^{n+1} \in \widehat{C}^{AR}_{ \alpha}\left(\text{mask}({X}^{n+1},m)\right)\mid M^{n+1}=m) \geq 1-\alpha.$$

\end{proof}

\begin{algorithm}[hbt]
   \caption{Acceptance-rejection corrected CP with pre-imputation of Calibration}
   \label{alg:ar_cif}
   \textbf{Input:} Calibration set $\mathrm{Cal}=\{(\widetilde{X}^{i}, Y^{i})\}_{i=1}^n$, generalized distributional imputation $\varphi$, predictive model $\hat{f}$, nonconformity score function $s_{\hat{f}}(\tilde{x}, y)$, likelihood estimation function $\hat{\omega}_m(\mask(x,m),y)$ for $\omega_m(x_{\text{obs}(m)},y)$ for $\omega_m(x_{\text{obs}(m)},y)$ with its maximum value $\omega_{{max}}$, significance level $\alpha$, test point $\widetilde{X}^{n+1}$.  The imputation function $\varphi$ and the predictive model $\hat{f}$ are trained on an independent training set. \\
   \textbf{Output:} Prediction set $\widehat{C}^{AR}_{ \alpha}(\widetilde{X}^{n+1})$.

   \begin{algorithmic}[1]
   \STATE $\mathcal{I}_{AR}=\{\}$
   \FOR{$k \in \{1,...,n\}$}
        \STATE Pre-impute Calibration point $(\varphi(\widetilde{X}^{k}, Y^{k}),  Y^{k})$
       \STATE Sample $U^{k} \sim U(0, 1)$
       \STATE Compute $q^{k} = \frac{\hat{\omega}_m(\mask (\varphi(\widetilde{X}^{k}, Y^{k}),M^{n+1}), \;Y^{k})}{\omega_{\mathrm{max}}}$
       \IF{$U^{k} < q^{k}$}
           \STATE $\mathcal{I}_{AR}=\mathcal{I}_{AR} \cup \{k\}$
           \STATE $\hat{X}^k = \varphi(\widetilde{X}^{k}, Y^{k})$
       \ENDIF
   \ENDFOR 
   \STATE $\widetilde{\mathrm{Cal}}=\{(\hat{X}^i,{Y}^i)\}_{i \in \mathcal{I}_{AR}}$
   \textcolor{blue}{// $|\mathcal{I}_{AR}|$ points are accepted such that $\{(\hat{X}^i,{Y}^i)\}_{i \in \mathcal{I}_{AR}} \subseteq \{(\varphi(\widetilde{X}^{i}, Y^{i}),  Y^{i})\}_{i=1}^n$}
   
   \STATE The updated dataset $\widetilde{\mathrm{Cal}}_{M^{n+1}} = \{\text{mask}(\hat{X}^{i},{M^{n+1}}), {Y}^{i})\}_{i \in \mathcal{I}_{AR}}$.
   \textcolor{blue}{// Mask $\widetilde{\mathrm{Cal}}$ with the missingness of test point $M^{n+1}$}
   \STATE $\forall i \in \mathcal{I}_{AR}$, denote $\hat{S}_{M^{n+1}}^i=s_{\hat{f}}(\text{mask}(\widetilde{X}^{i},{M^{n+1}}), {Y}^{i})$

   \STATE Compute the AR-corrected prediction set:
   \begin{align}
       &\widehat{C}^{AR}_{ \alpha}(\widetilde{X}^{n+1}) := \bigg\{ y \in \mathcal{Y}: s_{\hat{f}}(\widetilde{X}^{n+1}, y)\leq \mathcal{Q}_{1-\alpha} \left(\sum_{i \in \mathcal{I}_{AR}} \delta_{\hat{S}^i_{M^{n+1}}} + \delta_{\infty}\right)\bigg\}.\label{eq:ar_prediction_set}
   \end{align}
      \end{algorithmic}
\end{algorithm}


     

\vspace{1em}
\begin{lemma}[Bayes Classification Risk and Total Variation Distance]
\label{lemma:bayes_tv_equivalence}
Let $(\Omega, \mathcal{F})$ be a measurable space,
and let $\mathbb{P}$ and $\mathbb{Q}$ be two probability measures on it.
Assume there exists a $\sigma$-finite measure $\mu$ dominating both,
and denote the corresponding densities by
$p = \frac{d\mathbb{P}}{d\mu}$ and $q = \frac{d\mathbb{Q}}{d\mu}$.

Define the joint distribution $\Pi$ on $\Omega \times \{0,1\}$ by
\[
\Pi(dx,dc)
= \tfrac{1}{2}\,\mathbb{P}(dx)\,\delta_{1}(dc)
+ \tfrac{1}{2}\,\mathbb{Q}(dx)\,\delta_{0}(dc),
\]
where $\delta_c$ denotes the Dirac measure concentrated at $c$.
Then the marginal distribution of $X$ is
\[
\mathbb{D} := \tfrac{1}{2}(\mathbb{P} + \mathbb{Q}).
\]

For any measurable classifier $\mathsf{c}:\Omega\to\{0,1\}$,
define the (0–1) classification risk as
$
\mathcal{R}(\mathsf{c})
:= P_{(X,C)\sim \Pi}[\mathsf{c}(X)\neq C]
$, and the Bayes-optimal risk is
$
\mathcal{R}^\star := \inf_{\mathsf{c}} \mathcal{R}(\mathsf{c}).
$
Then we have:
\[
\mathcal{R}^\star
= \frac{1}{2} - \frac{1}{2} d_{\mathrm{TV}}(\mathbb{P}, \mathbb{Q}),
\]
where
\[
d_{\mathrm{TV}}(\mathbb{P}, \mathbb{Q})
:= \frac{1}{2}\int_\Omega |p(x) - q(x)|\,d\mu(x).
\]
\end{lemma}
\begin{proof}
The posterior probability that a sample $x$ is drawn from $\mathbb{P}$ (i.e., $C=1$)
given $x$ is defined as
\[
\eta(x) := P(C=1 \mid X=x)
= \frac{\tfrac{1}{2} p(x)}{\tfrac{1}{2} p(x) + \tfrac{1}{2} q(x)}
= \frac{p(x)}{p(x) + q(x)}.
\]
The marginal distribution of $X$ under the joint law $\Pi$ is
\[
\mathbb{D}(dx) = \tfrac{1}{2}(\mathbb{P} + \mathbb{Q})(dx)
= \tfrac{1}{2}(p(x) + q(x))\, d\mu(x).
\]
The Bayes-optimal classifier is
\[
\mathsf{c}^\star(x) = \mathbf{1}\{\eta(x) > 1/2\},
\]
and the corresponding Bayes classification risk is
\[
\mathcal{R}(\mathsf{c}^\star)
= P_{(X,C)\sim \Pi}[\mathsf{c}^\star(X) \neq C]
= \mathbb{E}_{X \sim \mathbb{D}}
   \big[ \min(\eta(X),\, 1 - \eta(X)) \big].
\]
Substituting the definition of $\eta(x)$ yields
\[
\min\!\left( \frac{p(x)}{p(x)+q(x)},\, \frac{q(x)}{p(x)+q(x)} \right)
= \frac{\min\{p(x), q(x)\}}{p(x)+q(x)}.
\]
Therefore,
\[
\mathcal{R}(\mathsf{c}^\star)
= \int_\Omega \tfrac{1}{2}(p(x)+q(x))
     \frac{\min\{p(x),q(x)\}}{p(x)+q(x)}\, d\mu(x)
= \tfrac{1}{2}\int_\Omega \min\{p(x), q(x)\}\, d\mu(x).
\]
Finally, recall that the total variation distance satisfies
\[
d_{\mathrm{TV}}(\mathbb{P},\mathbb{Q})
= \tfrac{1}{2}\int_\Omega |p(x)-q(x)|\,d\mu(x)
= 1 - \int_\Omega \min\{p(x), q(x)\}\,d\mu(x),
\]
which gives
\[
\mathcal{R}(\mathsf{c}^\star)
= \tfrac{1}{2} - \tfrac{1}{2}\, d_{\mathrm{TV}}(\mathbb{P}, \mathbb{Q}).
\]
\end{proof}

\subsection{Proof of \Cref{prop:imperfect_densityratio}}
\label{proof:imperfect_densityratio}
We fix an $m \in \{0,1\}^d$, let $({X}^{\tilde{\omega}}_{\obs (m)},{Y}^{\tilde{\omega}})$ a random vector following $d\tilde{P}_m:=\tilde{\omega}_m(\hat{X}_{\obs (m)},Y)\cdot dQ_m$, which is a well defined new distribution via a Radon–Nikodym derivative with respect to $Q_m$, using a normalized weighting function $\tilde{\omega}_m$. We define for convenience that ${X}^{\tilde{\omega}}_{\mis (m)}:=\NA$. 

Since $\widehat{C} _{\alpha}^{\widehat{W}}$ is constructed on calibration dataset $Q_m$ with correction weights $\tilde{\omega} _m$, it means that the reweighted calibration distribution is equal to the distribution of $({X}^{\tilde{\omega}}_{\obs (m)},{Y}^{\tilde{\omega}})$. 
By construction, this means that if we use the same Weighted CP method to compute a prediction interval for ${X}^{\tilde{\omega}}$, then we have the exact valid coverage:
\begin{equation}
\label{eq:contructedxy}
    P(Y^{\tilde{\omega}} \in \widehat{C} _{\alpha}^{\widehat{W}}({X}^{\tilde{\omega}})) \geq 1-\alpha.
\end{equation}

Besides, we note that, by definition of likelihood ratio $\omega_m$, the real $({X}_{\obs (m)},Y)$ follows the distribution $P_m=\omega_m(\hat{X}_{\obs (m)},Y)\cdot Q_m$. Denote by $A$ the event of $Y$ being in the prediction interval $\widehat{C} _{\alpha}^W(\widetilde{X})$ conditionally on $M^{n+1}=m$, we have, by the basic property of total variation distance:
\begin{align*}
    & |P(Y^{n+1} \in \widehat{C} _{\alpha}^{\widehat{W}}(\widetilde{X}^{n+1})\mid M^{n+1}=m)-P({Y}^{\tilde{\omega}} \in \widehat{C} _{\alpha}^{\widehat{W}}({X}^{\tilde{\omega}})\mid M^{n+1}=m)| \\
    =& |P_{(X_{\obs(m)},Y)\sim P_m}(A)-P_{(X_{\obs(m)},Y)\sim\tilde{\omega}_m(\hat{X}_{\obs (m)},Y)\cdot dQ_m}(A)| \\
    =& |P_{(X_{\obs(m)},Y)\sim P_m}(A)-P_{(X_{\obs(m)},Y)\sim\tilde{P}_m}(A)| \\
    \leq& d_{TV}(P_m, \tilde{P}_m)\\
\end{align*}
Combine \Cref{eq:contructedxy} with the previous results, we have the result of \Cref{prop:imperfect_densityratio}.

\subsection{Regularity assumptions}
\label{sec:regularity}

\paragraph{Notation.}
Fix a mask $m$. $P_m$ and $\tilde P_m$ are probability measures on
$\mathcal X_{\mathrm{obs}(m)}\times\mathcal Y$.
Introduce the balanced mixture on
$(\mathcal X_{\mathrm{obs}(m)}\times\mathcal Y)\times\{0,1\}$:
\[
\Pi_m \;=\; \tfrac12\,\big(P_m\times\delta_1\big)\;+\;\tfrac12\,\big(\tilde P_m\times\delta_0\big),
\]
where $\delta_c$ denotes the Dirac mass at $c\in\{0,1\}$.
Let $\mathcal G\subset\{\,g:(\mathcal X_{\mathrm{obs}(m)}\times\mathcal Y)\to\{0,1\}\,\}$
be a hypothesis class. Define the $0$–$1$ loss
\[
\gamma\big(g,(x_{\mathrm{obs}(m)},y),c\big)=\mathbf 1\{\,g(x_{\mathrm{obs}(m)},y)\neq c\,\}\in[0,1],
\]
its risk
\[
R_\gamma(g):=\mathbb E_{((X_{\mathrm{obs}(m)},Y),C)\sim \Pi_m}\!\big[\gamma(g,(X_{\mathrm{obs}(m)},Y),C)\big],
\]
and the Bayes risk $R_\gamma^\star:=\inf_{g}R_\gamma(g)$.
By Lemma~\ref{lemma:bayes_tv_equivalence},
$R_\gamma^\star=\tfrac12-\tfrac12\,d_{\mathrm{TV}}(P_m,\tilde P_m)$.

We restrict the search for the classifier $g_m$ to a hypothesis class $\mathcal{G}_m$, from which it is obtained via empirical minimization of the $0$–$1$ loss defined above.

\begin{assumption}[Separability of $\mathcal{G}_m$]\label{ass:M}
There exists a countable $\mathcal G'\subset\mathcal G_m$ such that for every $g\in\mathcal G_m$,
there is a sequence $(g_k)\subset\mathcal G'$ with
$\gamma\big(g_k,(x_{\mathrm{obs}(m)},y),c\big)
\to \gamma\big(g,(x_{\mathrm{obs}(m)},y),c\big)$ for all $((x_{\mathrm{obs}(m)},y),c)$.
\end{assumption}

\begin{assumption}[VC-class]\label{ass:VC}
Let $\mathcal A_{\mathcal{G}_m} := \{\,\{(x_{\mathrm{obs}(m)},y) : g_m(x_{\obs(m)},y)=1\}\;:\; g_m\in{\mathcal{G}_m}\}$ 
be the collection of decision sets induced by ${\mathcal{G}_m}$.
We assume $\mathcal A_{{\mathcal{G}_m}}$ is a VC-class with finite dimension $V\ge1$.
\end{assumption}

\begin{assumption}[ERM over $\mathcal{G}_m$]\label{ass:ERM01}
Given $l$ i.i.d.\ samples from $\Pi_m$, let
$\hat g_m\in\arg\min_{g\in\mathcal G_m}\hat R_\gamma(g)$ be an empirical risk minimizer of $\gamma$ over $\mathcal G$.
\end{assumption}

\begin{assumption}\label{ass:realizable}
There exists $g_m^\star \in \mathcal{G}_m$ such that $R_\gamma(g_m^\star)=R_\gamma^\star$.
\end{assumption}

\begin{assumption}[Uniform $\beta$-stability of the training algorithm]\label{ass:stability}
Let $\check{Z}=(Z^1,\ldots,Z^l)\in\mathcal Z^l$ with $Z_i=(X^i_{\obs(m)},Y^i,C_i)$ i.i.d.\ $\sim \Pi_m$.
Let $\mathcal{T}$ be a learning algorithm that outputs a classifier $\hat{g}_m = \mathcal{T}(\check{Z}) \in \mathcal{G}_m$.
For $i \in \{1, \ldots, l\}$, define the dataset $\check{Z}^{(i)}$ by replacing $Z^i$ with an independent copy $Z^{\prime} \sim \Pi_m$, i.e.,  
$\check{Z}^{(i)} = (Z^1, \ldots, Z^{i-1}, Z^{\prime}, Z^{i+1}, \ldots, Z^l)$.
We say that the algorithm $\mathcal{T}$ is \emph{uniformly $\beta_l$-stable} (with respect to the loss $\gamma$) if there exists $\beta_l \geq 0$ such that for all $i \in \{1, \ldots, l\}$ and all $z = (x_{\obs(m)}, y, c) \in \mathcal{X}_{\obs(m)} \times \mathcal{Y} \times \{0,1\}$,
\[
\bigl| \gamma(\mathcal{T}(\check{Z}), z) - \gamma(\mathcal{T}(\check{Z}^{(i)}), z) \bigr| \leq \beta_l
\quad \text{almost surely (w.r.t.\ } \Pi_m^{\,l+1}, \text{the joint distribution of } \check{Z}, Z^{\prime}).
\]
\end{assumption}
These assumptions are inspired by classical statistical learning theory \citep{Massart_2006,van1996weak,vapnik2013nature} and developments on algorithmic properties and stability \citep{bousquet2002stability, soloff2024bagging}. For example, kernel methods, regularized logistic regression, and neural networks with bounded weights satisfy separability and finite VC-dimension assumptions. Tree-based ensembles such as Random Forests or Gradient Boosted Trees also satisfy these assumptions when the number and depth of trees are bounded. Assumption~\ref{ass:stability} (Uniform $\beta_l$-stability) is trivially satisfied with zero-one loss. For generalizations, see \citet{bousquet2002stability, soloff2024bagging}.


\begin{algorithm}[hbt!]
\caption{Construction of datasets from $P_m$ and $\tilde{P}_m$ via acceptance-rejection sampling}
\label{alg:generate_P}
Suppose we are given $J$ i.i.d. samples $\{(\check{X}^i, \check{Y}^i, \check{M}^i)\}_{i=1}^J$. We first apply the same pre-imputation function used in the main text to obtain a fully observed dataset that follows distribution $Q$ by construction.

We then apply the fixed mask $m$ to each data point, yielding masked samples $\{(\widetilde{\check{X}}^i, \check{Y}^i)\}_{i=1}^J:=\{(\mask({\check{X}}^i,m), \check{Y}^i)\}_{i=1}^J$, which follow the distribution $Q_m$. 

\medskip

To generate i.i.d. samples from $\tilde{P}_m$, defined via the likelihood ratio $\tilde{\omega}_m$ as $d\tilde{P}_m := \tilde{\omega}_m \, dQ_m$, we apply an acceptance-rejection scheme with proposal distribution $Q_m$:

\begin{enumerate}
    \item Select a candidate $Z^i = (\widetilde{\check{X}}^i, \check{Y}^i)$ from the dataset, and draw $U \sim \mathrm{Unif}(0,1)$ independently.
    \item Accept $Z^i$ if $U \leq \tilde{\omega}_m(Z^i) / \tilde{K}$, where $\tilde{K} \geq \sup_z \tilde{\omega}_m(z)$.
\end{enumerate}

Each accepted point is an exact draw from $\tilde{P}_m$, and repeating the procedure yields an i.i.d. dataset of size $l_2$ from $\tilde{P}_m$.

\medskip

To construct the dataset from $P_m$, we select all samples from the original collection $\{(\check{X}^i, \check{Y}^i, \check{M}^i)\}_{i=1}^J$ that satisfy $\check{M}^i = m$. These points follow $P_m$ by definition.

\medskip

This procedure yields two independent datasets, each consisting of i.i.d. samples from $P_m$ and $\tilde{P}_m$, respectively.
\end{algorithm}

\subsection{Proof of \Cref{prop:estimablebound}}
\label{proof:prop3}
(1). Suppose we have another dataset of size $l$ which is denoted by $Z_c:=\{(\check{X}^i,\check{Y}^i,\check{C}^i)\}_{i=1}^{l}$ and an estimator of likelihood ratio $\hat{\omega}_m$. Our goal is to provide an estimation for $d_{TV}(P_m,\ \tilde{\omega}_m(\hat{X}_{\obs (m)},Y)\cdot Q_m)$.

Since $\tilde{\omega}_m$ is the normalized estimator of likelihood ratio, $\tilde{P}_m:=\tilde{\omega}_m(\hat{X}_{\obs (m)},Y)\cdot Q_m$, which defines a valid probability distribution.


Under Assumptions~\ref{ass:M} and \ref{ass:VC},  by Corollary~3 in \citet{Massart_2006}, we have
\[
\mathbb{E}_{Z_c}\!\big[R_\gamma(\hat g_m)\big]-R_\gamma^\star\;\le\; K\sqrt{\tfrac{V_{\mathcal{G}}}{l}},
\]
where the $K$ is an universal constant independent of the sample size and the distribution and $V_{\mathcal{G}}$ is the VC-dimension of the hypothesis class $\mathcal{G}$ as in \Cref{ass:VC}. 

(2). Consider the construction of prediction interval $\hat{C}^{W}_\alpha$, which is independent of the dataset $Z_c$. Therefore, we can apply the \Cref{prop:imperfect_densityratio} and have 
\begin{align}
\label{eq:corollary2}
P\bigg(Y^{n+1}\in\hat{C}^{W}_\alpha(\widetilde{X}^{n+1})\mid M^{n+1}=m\bigg)-(1-\alpha)&\geq 
-d_{TV}(P_m,\tilde{P}_m)\\
&=-(1-2R^{\star}_\gamma)\,,
\end{align}
where the last equality follows from~\Cref{lemma:bayes_tv_equivalence}.
Since $\hat{C}^{W}_\alpha$ is constructed independently of $Z_c$, the right-hand side of \cref{eq:corollary2} can be rewritten as
\begin{align}
\label{eq:corolast}
&P(Y^{n+1}\in\hat{C}^{W}_\alpha(\widetilde{X}^{n+1})\mid M^{n+1}=m)-(1-\alpha)\\
\geq&-(1-2\, R^{\star}_\gamma) =-1 +  2\, R^{\star}_\gamma \nonumber\\
   \geq&-1+2\mathbb{E}_{Z_c}[R_\gamma (\hat{g}_m)]-K\sqrt{\frac{V_{\mathcal{G}_m}}{l}}\,,\nonumber
\end{align}
where \(K > 0\) is a universal constant that may absorb absolute numerical constants such as 2. Under \Cref{ass:stability}, we can apply McDiarmid’s inequality \citep{mcdiarmid1989method} to obtain the following bound: for any $\epsilon>0$, 
$P_{Z_c}\bigg(R_\gamma (\hat{g}_m)-\mathbb{E}_{Z_c}[R_\gamma (\hat{g}_m)])\geq \epsilon\bigg) \leq \exp({-\frac{2 \epsilon^2}{l \beta_l ^2}})$. Since all the risk are bounded by 0 and 1, we apply Hoeffding's inequality and get
$P_{Z_c}(\hat{R}_\gamma(\hat{g}_m)-R_\gamma(\hat{g}_m)\geq \epsilon) \leq \exp(-\frac{2 \epsilon^2 l^2}{l})$.

Replacing $\epsilon $ by $\epsilon/2$, we have 
\begin{equation}
    P_{Z_c}(|\hat{R}_\gamma(\hat{g}_m)-\mathbb{E}_{Z_c}[R_\gamma(\hat{g}_m)]|\leq \epsilon) > 1-\exp(-\frac{\epsilon^2}{2l\beta_l^2})-\exp(-\frac{\epsilon^2 l}{2}).
\end{equation}
Substituting into \Cref{eq:corolast}, we obtain the desired result.

\section{Impact of Estimator Quality on Coverage}
\label{appendix:estimator quality}

In this section, we study the practical implications of density ratio estimation quality on the mask-conditional coverage of weighted conformal prediction. The analysis builds on synthetic data where the true density ratios are analytically tractable.

\subsection{Synthetic Data Generation}

We consider the same synthetic setup as in  \cref{sec:experiments}. Let $X \in \mathbb{R}^d$ be a $d$-dimensional Gaussian random variable and $Y$ be the linear response with additive Gaussian noise:
\begin{align*}
X &\sim \mathcal{N}(\mu_X, \Sigma_X), \\
Y \mid X &= \beta^\top X + \varepsilon, \quad \varepsilon \sim \mathcal{N}(0, \sigma^2).
\end{align*}
The joint distribution $(X,Y)$ is Gaussian with mean and covariance:
\begin{align*}
\mu_{XY} &= \begin{bmatrix} \mu_X \\ \beta^\top \mu_X \end{bmatrix}, \\
\Sigma_{XY} &= \begin{bmatrix} \Sigma_X & \Sigma_X \beta \\ \beta^\top \Sigma_X & \beta^\top \Sigma_X \beta + \sigma^2 \end{bmatrix}.
\end{align*}

We simulate Missing Completely at Random (MCAR) by randomly masking each feature $X_j$ independently with probability $0.5$. Let $M \in \{0,1\}^d$ denote the missingness indicator.

By Gaussian conditioning, for each missing pattern $M=m$, the conditional distribution of the missing covariates given observed covariates and $Y$ is:
\begin{align*}
X_{\text{mis}} \mid X_{\text{obs}}, Y &\sim \mathcal{N}(\mu_{\text{mis} \mid \text{obs}, Y}, \Sigma_{\text{mis} \mid \text{obs}, Y}),
\end{align*}
where:
\begin{align*}
\mu_{\text{mis} \mid \text{obs}, Y} &= \mu_{\text{mis}} + \Sigma_{\text{mis}, [X_{\text{obs}}, Y]} \Sigma_{[X_{\text{obs}}, Y]}^{-1} \left(\begin{bmatrix} X_{\text{obs}} \ Y \end{bmatrix} - \mu_{[X_{\text{obs}}, Y]} \right), \\
\Sigma_{\text{mis} \mid \text{obs}, Y} &= \Sigma_{\text{mis}, \text{mis}} - \Sigma_{\text{mis}, [X_{\text{obs}}, Y]} \Sigma_{[X_{\text{obs}}, Y]}^{-1} \Sigma_{[X_{\text{obs}}, Y], \text{mis}}.
\end{align*}

\subsection{Perturbed Imputation Strategy}

To emulate estimation noise in imputation, we perturb the conditional distribution:
\begin{align*}
\varphi(\tilde{x}, y)_{\text{mis}} &\sim \mathcal{N}(\tilde{\mu}_{\text{mis}}, \tilde{\Sigma}_{\text{mis}}),
\end{align*}
with:
\begin{align*}
\tilde{\mu}_{\text{mis}} &= \alpha \cdot \mu_{\text{mis} \mid \text{obs}, Y} + \Delta, \\
\tilde{\Sigma}_{\text{mis}} &= \Sigma_{\text{mis} \mid \text{obs}, Y} + \tau I,
\end{align*}
where $\alpha > 0$, $\Delta \in \mathbb{R}^{|\text{mis}|}$ is a bias term, and $\tau > 0$ is variance inflation.

As discussed in the main text, the key ingredient for weighted conformal prediction is the likelihood ratio:
\begin{align*}
\omega_m(x_{\text{obs}}, y) = \frac{dP_m(x_{\text{obs}}, y)}{dQ_m(x_{\text{obs}}, y)},
\end{align*}
where $P_m$ is the true mask-conditional distribution of $(X,Y)|M=m$, and $Q_m$ is the induced distribution from imputed calibration data. In this setup, both are computable: $P_m$ is the original Gaussian, and $Q_m$ is a mixture Gaussian from the perturbed imputation. Under our perturbed imputation strategy, all the terms are tractable.

\subsection{Empirical Evaluation}

We perform controlled simulations under fixed mask $m = [1,1,1,0,0]$ with $d=5$, $\mu_X = \mathbf{1}_d$, and covariance $\Sigma_X = \rho \mathbf{1}_d \mathbf{1}_d^\top + (1 - \rho) I$, $\rho = 0.8$. We generate 500 training, 200 calibration, and 100 test samples. For perturbation, we set $\Delta=0.5, \alpha=1.5, \tau=0.2$. The target coverage is 90\%. We then evaluate the following:
\begin{enumerate}
\item \textbf{No correction:} miscoverage due to distribution shift.
\item \textbf{Oracle correction:} use exact $\omega_m$.
\item \textbf{Estimated correction:} learn an estimator $\hat{\omega}_m$ of $\omega_m$ using \cref{alg:likelihood_ratio_estimation}.
\item \textbf{Perturbed estimation:} apply controlled noise to $\hat{\omega}_m$ test robustness of different qualities of density estimators (see \Cref{appendix:controlled_noise} for details).
\end{enumerate}

Each setting is repeated 100 times to compute coverage statistics and its sensitivity to estimator quality. We find that:
\begin{itemize}
\item \textbf{Uncorrected CP} underestimates coverage (86\%) due to test/calibration shift.
\item \textbf{Oracle-corrected CP} achieves near-perfect coverage (89--90\%).
\item \textbf{Estimated weights} by using \Cref{alg:likelihood_ratio_estimation} achieve our the desired coverage, though showing conservative behavior (93\% coverage).
\item \textbf{Estimator degradation} leads to systematic under-coverage as correlation between estimated and true weights drops (see \Cref{appendix:controlled_noise} for details).
\end{itemize}

\subsection{Impact of Controlled Perturbation in Density Ratio Estimation}
\label{appendix:controlled_noise}

To examine empirically the robustness of our method under imperfect density ratio estimation, we introduce controlled perturbations. This complements the theoretical and empirical findings discussed in the main text and addresses concerns regarding the sensitivity of our approach to density ratio quality.

\paragraph{Noise Injection into Estimated Density Ratio.} In our implementation, density ratios are estimated using Algorithm~2 (a histogram-based gradient boosting classifier). To simulate various levels of estimation quality, we add Gaussian noise to the logarithm of the estimated density ratios:
\begin{align*}
\log \tilde{\omega}_m = \log \hat{\omega}_m + \varepsilon, \quad \varepsilon \sim \mathcal{N}(0, \sigma^2),
\end{align*}
then exponentiate back to obtain noisy $\tilde{\omega}_m$. This ensures the perturbation is multiplicative and realistically models deviations in the ratio space.

\paragraph{Empirical Evaluation.} For each of 100 repetitions, we pertube the estimator of density ratio and we evaluate the quality of density ratio estimator on calibration data using the Pearson correlation coefficient between true and estimated density ratio values. We then measure the mask-conditional coverage obtained from our weighted CP method. Building on this, for each repetition, we have a different quality of density ratio estimator (in terms of Pearson correlation) and an associated coverage. We then draw the scatter plot of the correlation against empirical coverage in \Cref{fig:correlation_coverage}. We observe a strong positive relationship: when correlation exceeds 0.3, coverage stabilizes around the nominal level of 90\%, demonstrating the method’s robustness to moderate estimation error.

\begin{figure}[h]
\centering
\includegraphics[width=0.7\textwidth]{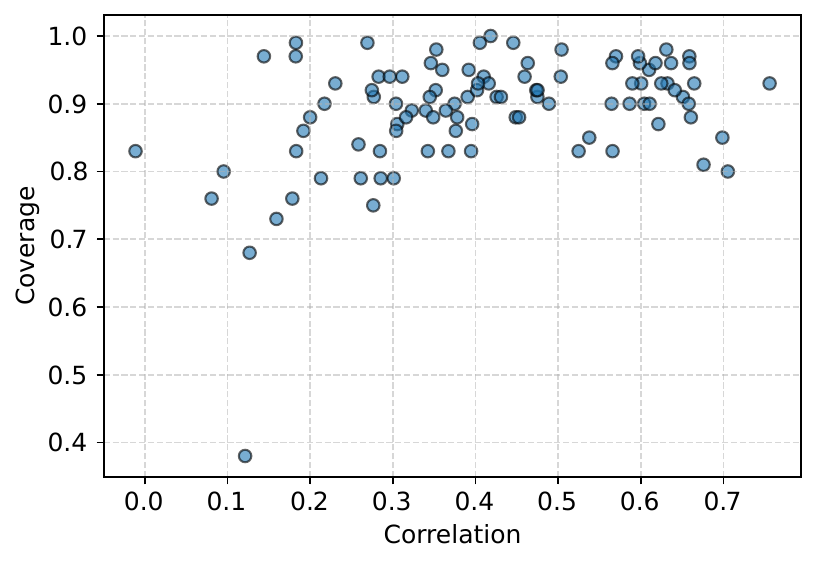}
\caption{Coverage vs. Correlation of estimated and true density ratios. Each point represents one experiment. High correlation yields near-nominal coverage.}
\label{fig:correlation_coverage}
\end{figure}

\subsection{Evaluation of theoretical Miscoverage Bound.} 
\label{exp:imperfect estimation}
Our goal is to empirically verify the validity of Proposition~\ref{prop:estimablebound}, which asserts that the miscoverage due to imperfect density ratio estimation can be upper bounded by quantities estimated from data. Specifically, we construct the distributions \(P_m\) and \(\tilde{P}_m\) using an auxiliary dataset of 500 i.i.d. samples, train a classifier $\hat{g}_m$ to distinguish them, and evaluate whether the resulting classification error yields a valid upper bound on the observed miscoverage.

By \Cref{proof:prop3}, we have:
\begin{align*}
 &P(Y \in \widehat{C}_{\alpha}^{\widetilde{W}}(\tilde{X}) \mid M=m) -90\% \geq- (1 - 2\widehat{\mathcal{R}}(\hat{g}_m)). \\
 \leftrightarrow & 90\% - P(Y \in \widehat{C}_{\alpha}^{\widetilde{W}}(\tilde{X}) \mid M=m) \leq (1 - 2\widehat{\mathcal{R}}(\hat{g}_m)).
\end{align*}

 For each repetition of experiment, we can compute a miscoverage by \[\text{Miscoverage}=(90\%-\text{True coverage}).\] We can also compute the empirical error $\widehat{\mathcal{R}}(\hat{g}_m)$ of the classifier between $P_m$ and $\tilde{P}_m$ by using an extra dataset , and the empirical miscoverage $(1 - 2\widehat{\mathcal{R}}(\hat{g}_m))$. Ideally, we should observe that all the miscoverage is controlled by $(1 - 2\widehat{\mathcal{R}}(\hat{g}_m))$. In \Cref{fig:error_bound_miscoverage}, we compare empirical miscoverage with the theoretical bound $(1 - 2\widehat{\mathcal{R}}(\hat{g}_m))$. Most observed miscoverage values lie below the red line $y=x$, verifying that the bound is valid.

\begin{figure}[h]
\centering
\includegraphics[width=0.7\textwidth]{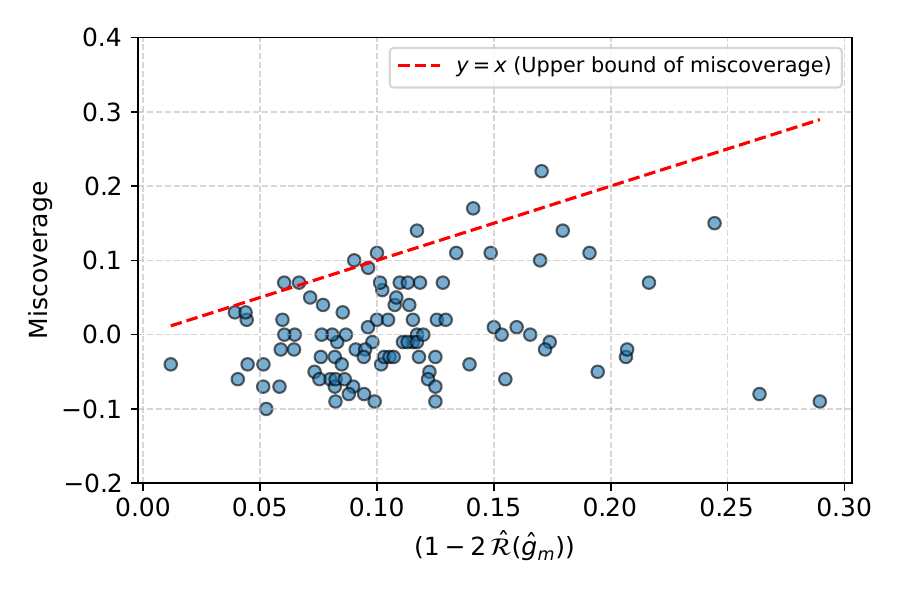}
\caption{Miscoverage vs estimated bound. Points below $y=x$ confirm the bound’s conservativeness.}
\label{fig:error_bound_miscoverage}
\end{figure}

Together, these findings demonstrate that our method is empirically robust to imperfect density ratio estimation, and theoretical bounds on miscoverage can be derived and validated in practice.

\section{Detailed experiments}
\label{appendix:comp}
\subsection{Running time comparison}
\label{app:running time}
We compare the running time of all 4 methods on the same dataset, i.e., 100 points for each missingness pattern. All experiments are running on the same server with CPU Intel XeonGold 6252 @ 2.10GHz. The running time of 4 methods on two datasets is shown in \Cref{tbl:runningtime}. Our proposed AR-Corrected method has a running time comparable to MDA-Exact, which is much faster than MDA-Nested and weighted CP. This aligns with our expectation since the MDA-Nested requires modifying test points and calibration points differently, and the weighted method requires using a "grid search" like method to evaluate the prediction interval, which cannot be solved explicitly. Therefore, our proposed acceptance rejection-based method can be a good candidate for substituting weighted conformal prediction, especially when the weighted CP interval has no simple and explicit form.

\begin{table}[hbt!]
\caption{Comparison of running time of 4 methods.}
\label{tbl:runningtime}
\vskip 0.15in
\begin{center}
\begin{footnotesize} 
\begin{sc}
\setlength{\tabcolsep}{4pt}
\renewcommand{\arraystretch}{1.1} 
\begin{tabular}{lccc}
\toprule
\textbf{Method} & \textbf{Synthetic dataset} & \textbf{Concrete dataset}  \\
\midrule
MDA-Exact        & 17034s    & 2031s     \\
MDA-Nested       & 510932s   & 43225s   \\
Weighted CP  & 228638s   & 37983s  \\
ARC CP    & 21320s    & 4766s  \\
\bottomrule
\end{tabular}
\end{sc}
\end{footnotesize}
\end{center}
\vskip -0.1in
\end{table}

\subsection{Another real dataset}
\paragraph{Bike Sharing
Dataset from the UCI repository\citep{bike_sharing_275}}
\label{bike_sharing}
The target value is the total count of bikes rented during each hour. There are four continuous covariates and 12 discrete covariates. 5000 training data and 100 calibration data are selected. Another 100 test data for each pattern of missingness are generated under MCAR mechanism with 50\% missing values for each feature. Since we have four continuous covariates, 15 $(2^4-1)$ types of missingness patterns exist.

\begin{table}[H]
\caption{Comparison on Bike sharing dataset for different methods. The average coverage and interval width of 15 different missingness patterns are listed.}
\label{tbl:bikesharing}
\vskip -0.2in
\begin{center}
\begin{small}
\begin{sc}
\setlength{\tabcolsep}{6pt} 
\renewcommand{\arraystretch}{1.2}
\begin{tabular}{lcc}
\toprule
\textbf{Method} & \textbf{Avg. Cov} & \textbf{Avg.Width} \\
\midrule
MDA-Exact        & $92.9 \pm 0.7$            & $\infty$ \\
MDA-Nested       & $91.8 \pm 0.5$            & $336.3 \pm 2.3$ \\
\textbf{Weighted CP}  & $91.9 \pm 0.6$            & $\mathbf{306.4 \pm 2.2}$ \\
\textbf{ARC CP}    & $93.0 \pm 0.6$            & $\mathbf{309.9 \pm 3.1}$ \\
\bottomrule
\end{tabular}
\end{sc}
\end{small}
\end{center}
\vskip -0.1in
\end{table}

The results are shown in \Cref{tbl:bikesharing}, we can observe similar results: MDA-Exact may not work for certain masks, thus leading the average width of the interval to be infinite, and MDA-Nested over-estimates the uncertainty, whereas our methods provide a prediction set with shorter width while being mask-conditional valid.

\paragraph{MEPS19 dataset}
The MEPS19 dataset refers to the 2019 Medical Expenditure Panel Survey conducted by the Agency for Healthcare Research and Quality (AHRQ)\footnote{\url{https://meps.ahrq.gov/data_stats/download_data_files.jsp}}. It is a large-scale, nationally representative survey of families and individuals in the United States, collecting detailed information on healthcare utilization, expenditures, insurance coverage, and health status. The dataset contains both continuous and categorical covariates: in our setup, we use 5 continuous and 134 discrete covariates.

The target variable in our analysis, \textit{UTILIZATION\_reg}, is an aggregated measure of healthcare utilization, derived from MEPS variables describing the frequency and intensity of medical service use. Due to the wide range of the target values, applying the Weighted Conformal Prediction (Weighted CP) method is computationally challenging, as it requires evaluating the conformity score across a dense grid of possible outcomes. Therefore, we focus on comparing MDA-Exact, MDA-Nested, and our ARC-CP method.

We randomly select 1000 training samples and 500 calibration samples. The 5 continuous covariates are subject to missingness following a Missing Completely at Random (MCAR) mechanism. We repeat the experiment 50 times to estimate the average coverage and corresponding confidence intervals. The results, shown in \Cref{tbl:meps19}, indicate that all methods achieve the desired coverage, while our ARC-CP method yields the shortest average interval length.

\begin{table}[H]
\caption{Comparison on Bike sharing dataset for different methods. The worst-case coverage and interval width of 32 different missingness patterns are listed.}
\label{tbl:meps19}
\vskip -0.2in
\begin{center}
\begin{small}
\begin{sc}
\setlength{\tabcolsep}{6pt} 
\renewcommand{\arraystretch}{1.2}
\begin{tabular}{lcc}
\toprule
\textbf{Method} & \textbf{Coverage} & \textbf{Interval width} \\
\midrule
MDA-Exact        & $94.03-95.01$            & $33.0$ \\
MDA-Nested       & $94.29-95.15$            & $34.1$ \\
\textbf{ARC CP}    & $94.04-95.04$            & $\mathbf{32.08}$ \\
\bottomrule
\end{tabular}
\end{sc}
\end{small}
\end{center}
\vskip -0.1in
\end{table}
\paragraph{Communities and Crime dataset}
To further explore high-dimensional behavior, we added experiments on the Communities and Crime dataset \citep{redmond2002data}, a regression task with 122 real-valued features. We allow all features to be missing, such that  there are $2^{122}$ possible masks, and randomly sample 1000 masks to evaluate mask-conditional coverage for each. For ARC-CP, the worst-case mask-conditional coverage among these 1000 masks is 87.8\%; only 49 masks have coverage below 90\% (between 87.8\% and 90\%), indicating that ARC-CP remains robust and reasonably close to the target even in this challenging high-dimensional real-data setting. MDA-Exact always gives infinite intervals, while Nested and Weighted CP cannot finish in a reasonable time (one week).

\subsection{The necessity for correction}
\label{appendix:correction}
To show that our correction method can guarantee MCV effectively, we do experiments on the dataset in \Cref{section:weight correction} with and without applying the correction methods. The comparison for whether using the correction method in terms of missingness coverage and interval width is shown in \Cref{img:correction_app,img:correction_width_app}. Without correction, the worst mask conditional coverage is 89.1\%, showing that without a suitable weight/acceptance for adjusting original Calibration, the distribution shift results in an under-coverage of the true label.

\begin{figure}[H]
\begin{center}
\centerline{\includegraphics[width=.9\textwidth]{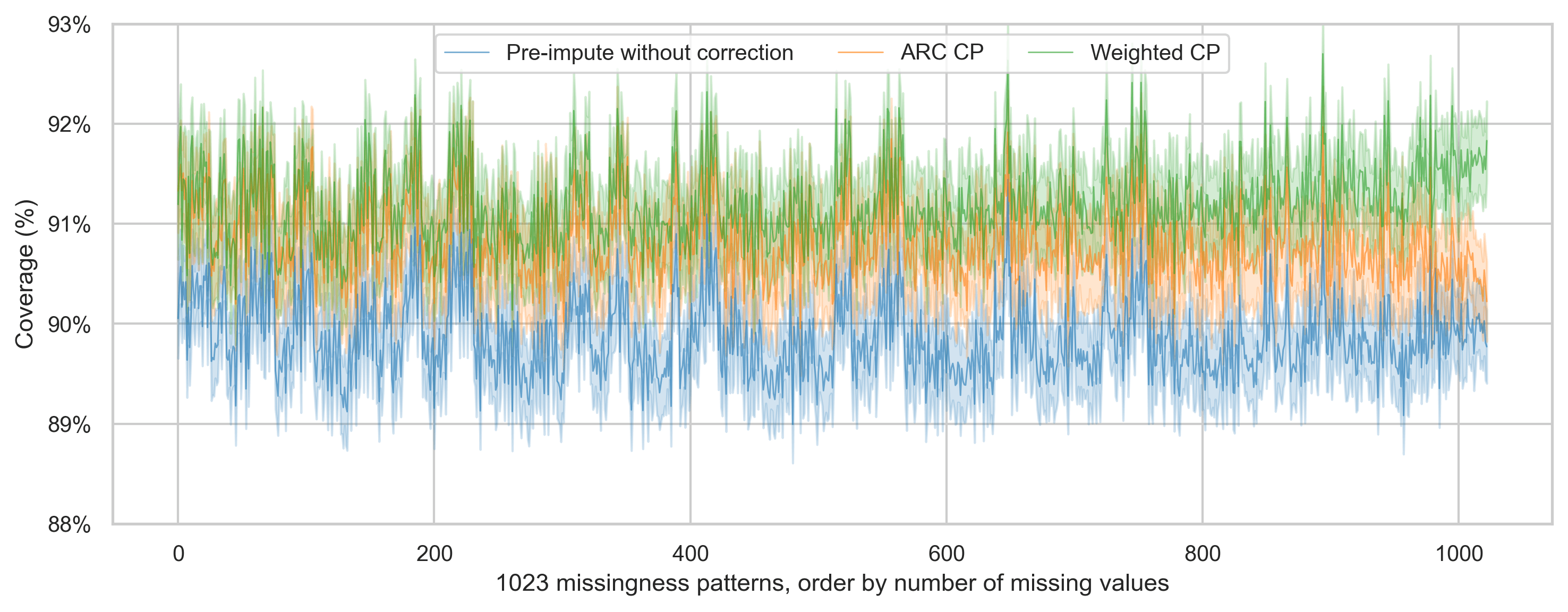}}
\caption{Synthetic dataset: Comparison of prediction coverage (mean and 99\% confidence interval) across 1023 missingness patterns, with and without the proposed correction method.}
\label{img:correction_app}
\end{center}
\end{figure}

\begin{figure}[H]
\begin{center}
\centerline{\includegraphics[width=.9\textwidth]{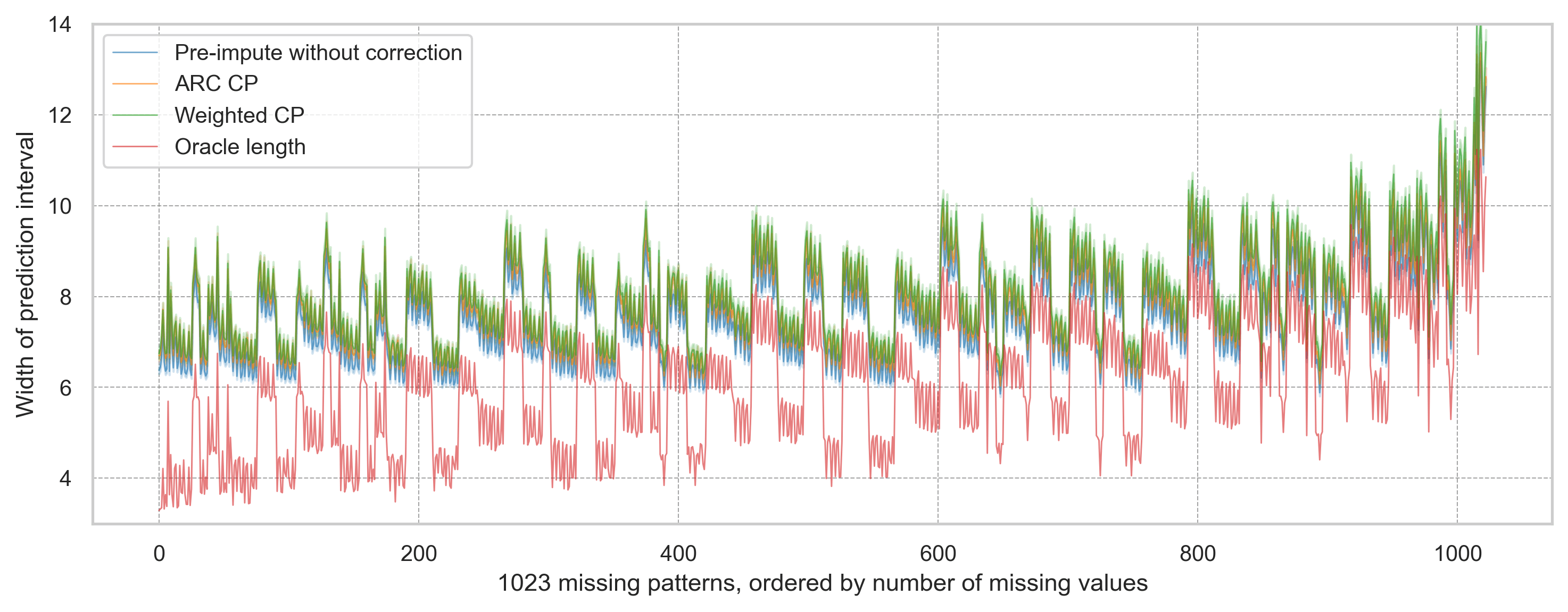}}
\caption{Synthetic dataset: Comparison of the width of prediction interval (mean and 99\% confidence interval) across 1023 missingness patterns, with and without the proposed correction method.
}
\label{img:correction_width_app}
\end{center}
\end{figure}

\subsection{Experiments of concrete dataset}
\label{section:compress}
We provide the missingness-wise coverage for concrete dataset, as shown in \Cref{img:concrete_coverage}.

\begin{figure}[H]
\begin{center}
\centerline{\includegraphics[width=0.6\textwidth]{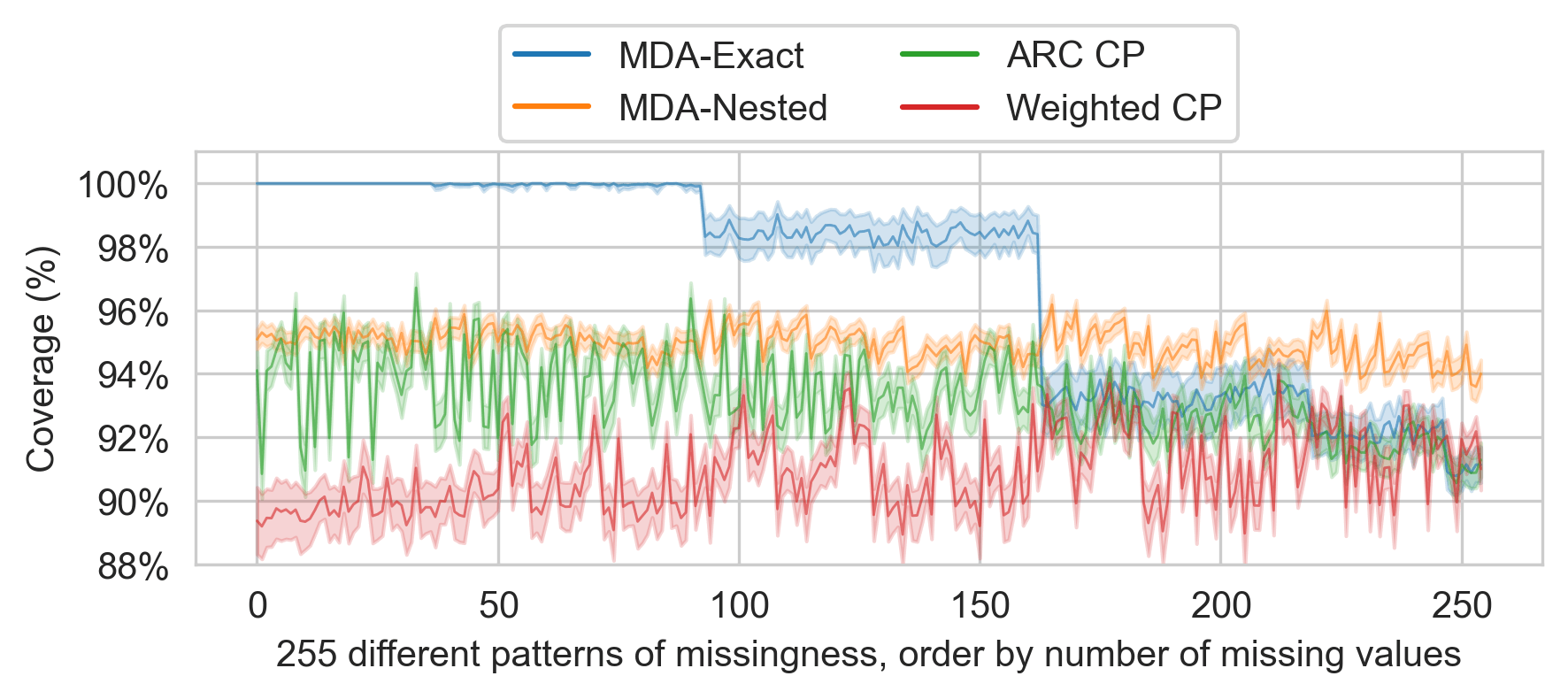}}
\caption{Concrete dataset: mean and 99\% confidence interval of prediction coverage with respect to 255 missingness patterns
}
\label{img:concrete_coverage}
\end{center}
\end{figure}

For MAR settings, the worst-case missingness-wise coverage and the marginal prediction interval width for the four methods are in \Cref{tbl:marconcrete}:

\begin{table}[h]
\centering
\caption{Worst-case coverage (95\% CI) and average prediction interval width over 15 masks, concrete dataset with 4 possible MAR missing covariates}
\label{tbl:marconcrete}
\begin{tabular}{lcccc}
\toprule
Mech. & MDA-Exact & MDA-Nested & Weighted CP & ARC CP \\
\midrule
MAR & 89.1\% $\pm$ 1.64\% / 29.1 & 90.8\% $\pm$ 1.43\% / 31.4 & 89.1\% $\pm$ 1.91\% / 29.8 & 90.6\% $\pm$ 1.42\% / 28.9  \\
\bottomrule
\end{tabular}
\vspace{-1em}
\end{table}

For self-masked MNAR settings, it is generally difficult to generate enough test points per pattern, we allow 2 covariates to be missing. The worst-case missingness-wise coverage and the marginal prediction interval width for the four methods are in \Cref{tbl:mnarconcrete}:

\begin{table}[h]
\centering
\caption{Worst-case coverage (95\% CI) and average prediction interval width over 15 masks, concrete dataset with 4 possible missing covariates}
\label{tbl:mnarconcrete}
\begin{tabular}{lcccc}
\toprule
Mech. & MDA-Exact & MDA-Nested & Weighted CP & ARC CP \\
\midrule
MNAR & 89.0\% $\pm$ 1.78\% / 30.1 & 88.7\% $\pm$ 1.76\% / 30.4 & 90.0\% $\pm$ 1.67\% / 30.8 & 89.0\% $\pm$ 1.83\% / 30.1  \\
\bottomrule
\end{tabular}
\vspace{-1em}
\end{table}

\subsection{Complementary of synthetic experiments for other settings}
\subsubsection{Less missing values}
We consider the same setting of the synthetic dataset, except we have 20\% missing values instead of 50\%. As we said, when there are many missing values in the dataset, and we want to achieve MCV for a test point with almost no missing values, the MDA-Exact could give an uninformative prediction interval due to the lack of sample points in the calibration dataset. Still, this method should work well when we have fewer missing values. The comparison of MDA-Exact, MDA-Nested, as well as our methods, is in \Cref{img:comp_interval_width_20missing,img:comp_interval_coverage_20missing}. We observe that MDA-Exact is less likely to give an uninformative prediction set, but both the MDA-Exact and Nested show a missingness conditional over-coverage when the test point has fewer missing values. Our method shows more stability in coverage for all patterns of missingness. The AR-corrected method reduces 10.7\% the average width of prediction interval than MDA-Nested for all patterns of missingness.

\begin{figure}[H]
\begin{center}
\centerline{\includegraphics[width=0.6\textwidth]{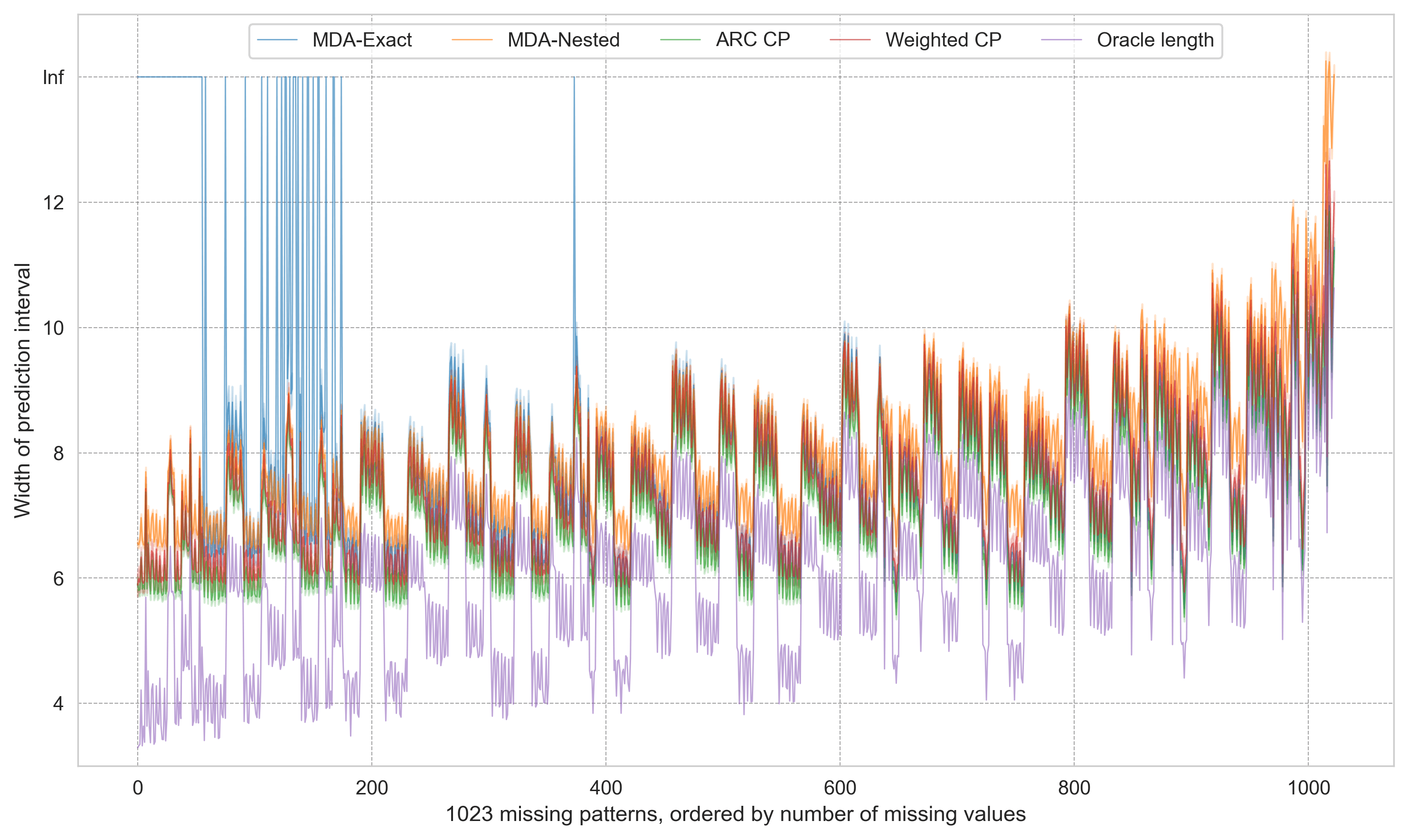}}
\caption{Synthetic dataset with 20\% missing values: mean and 99\% confidence interval of prediction interval width with respect to 1023 missingness patterns, comparison for whether using correction method.}
\label{img:comp_interval_width_20missing}
\end{center}
\end{figure}

\begin{figure}[H]
\begin{center}
\centerline{\includegraphics[width=0.7\textwidth]{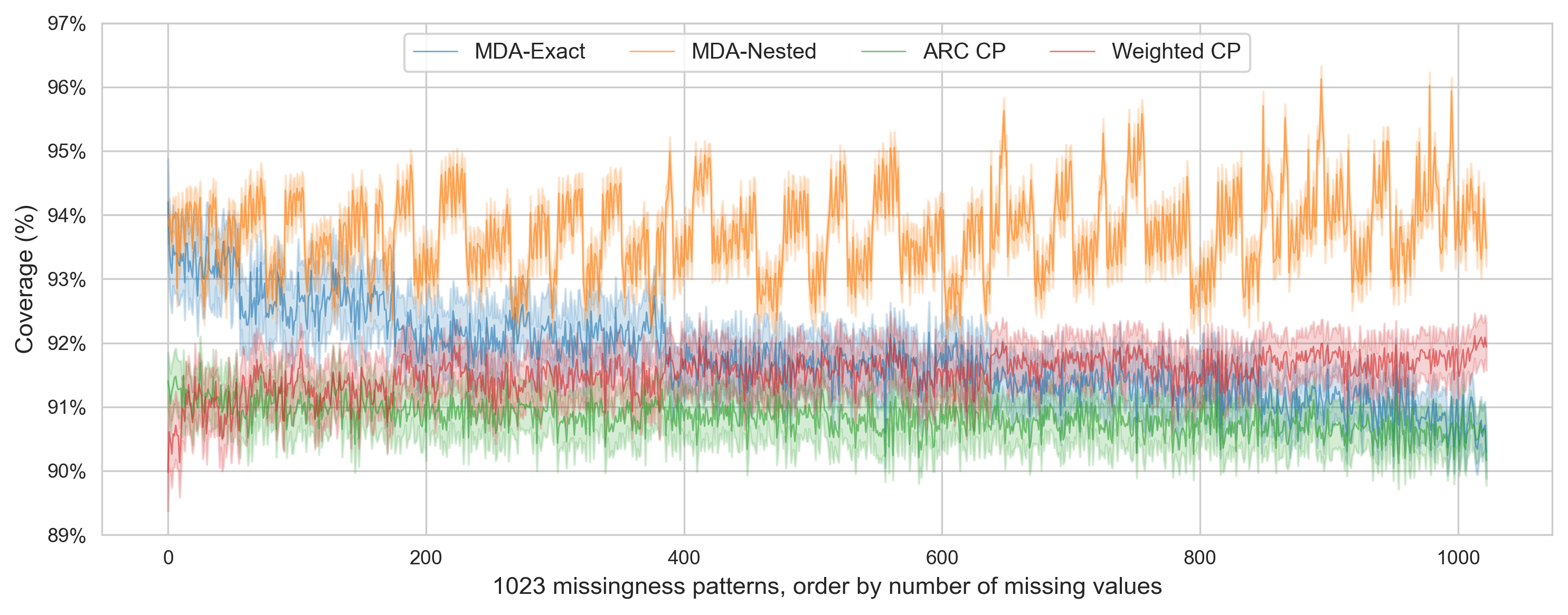}}
\caption{Synthetic dataset with 20\% missing values: mean and 99\% confidence interval of prediction coverage with respect to 1023 missingness patterns, comparison for whether using correction method.}
\label{img:comp_interval_coverage_20missing}
\end{center}
\end{figure}

\subsubsection{Synthetic dataset with independent features $\rho=0$}
Similarly to \citet{zaffran2024predictiveuncertaintyquantificationmissing}, we run another experiment on a dataset in which the features are independent ($\rho=0$), leading to an imputation that can not be better than the marginal expectation of the features. We also keep 20\% missing values. The same results are shown in \Cref{img:comp_interval_coverage_20missing_independent,img:comp_interval_width_20missing_independent}. We observe similar results: the AR-corrected method reduces 8.9\% the average width of prediction interval than MDA-Nested for all patterns of missingness.

\begin{figure}[H]
\begin{center}
\centerline{\includegraphics[width=0.8\textwidth]{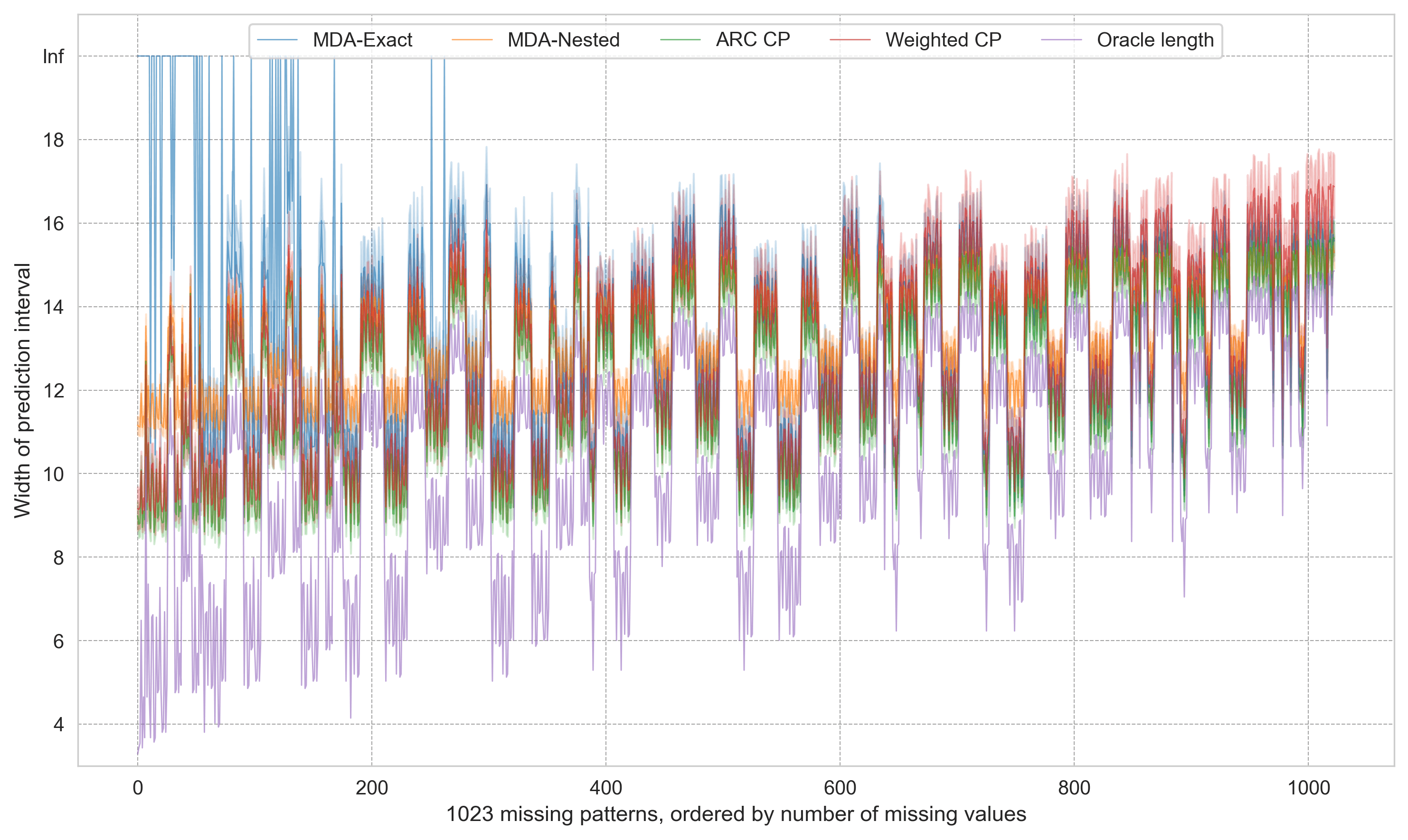}}
\caption{Synthetic dataset with 20\% missing values and independent features: mean and 99\% confidence interval of prediction interval width with respect to 1023 missingness patterns, comparison for whether using correction method.}
\label{img:comp_interval_width_20missing_independent}
\end{center}
\end{figure}

\begin{figure}[H]
\begin{center}
\centerline{\includegraphics[width=0.8\textwidth]{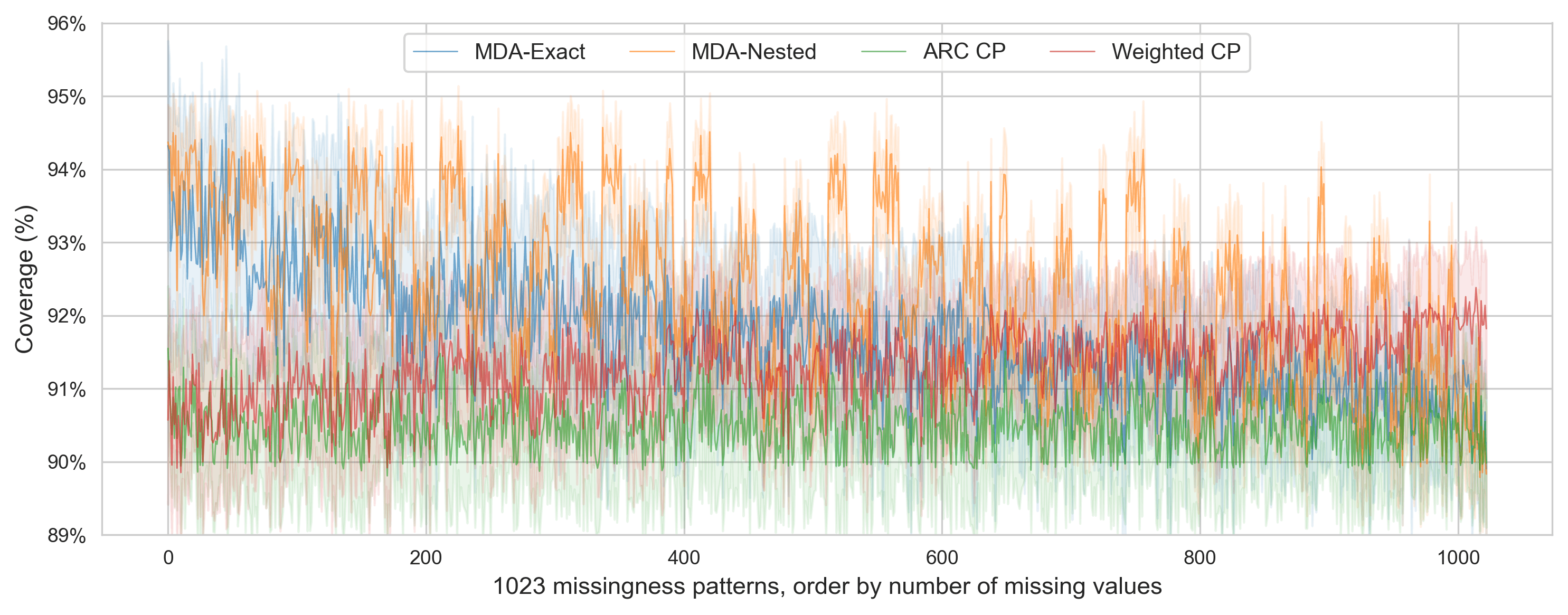}}
\caption{Synthetic dataset with 20\% missing values and independent features: mean and 99\% confidence interval of prediction coverage with respect to 1023 missingness patterns, comparison for whether using correction method.}
\label{img:comp_interval_coverage_20missing_independent}
\end{center}
\end{figure}

\subsection{Unsmoothed results for  \Cref{synthetic_width,synthetic_coverage}}
The unsmoothed results in the Experiments \Cref{sec:experiments} are shown in \Cref{synthetic_width_unsmoothed,synthetic_coverage_unsmoothed}.

\begin{figure}[ht]
\vskip 0.1in
\begin{center}
\centerline{\includegraphics[width=0.8\textwidth]{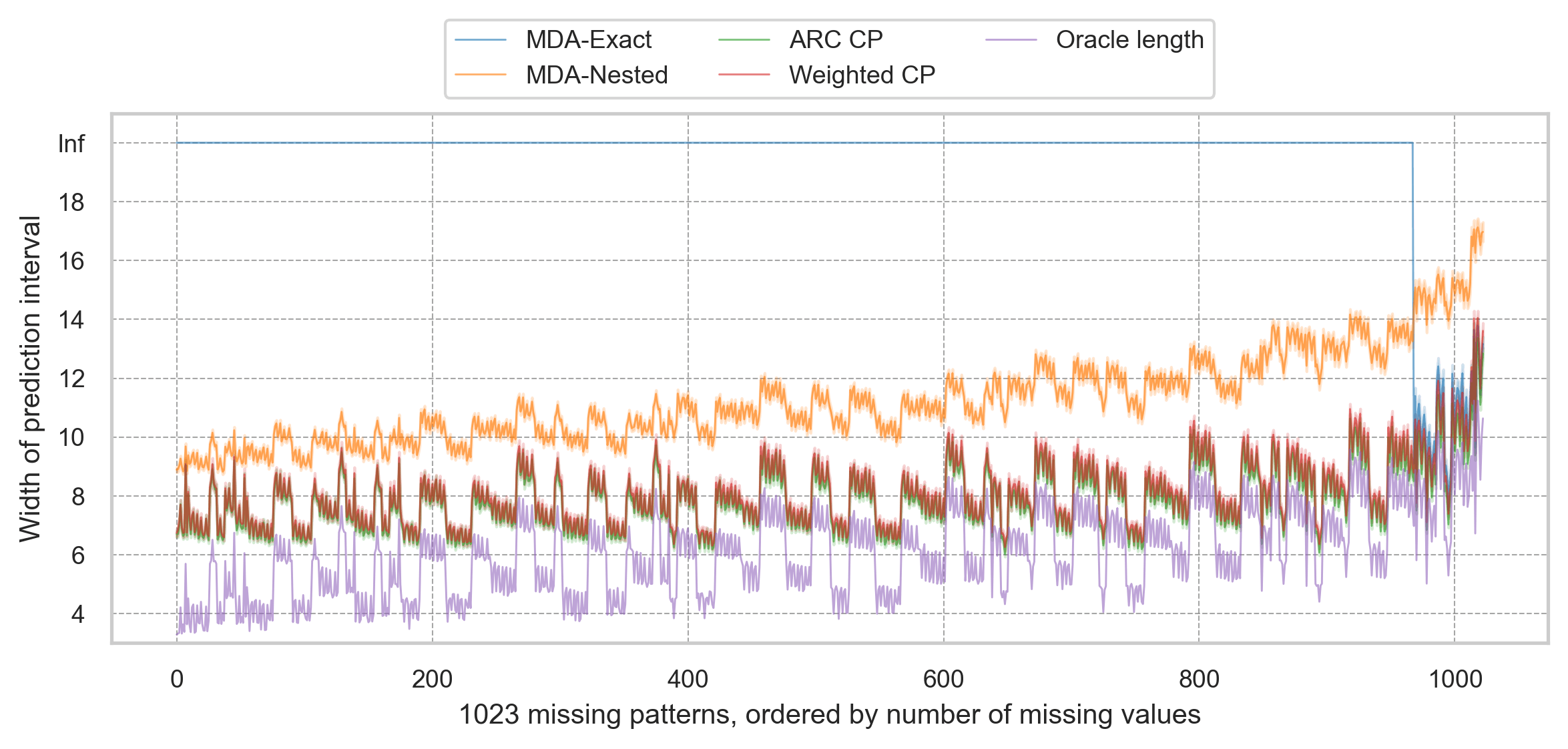}}
\caption{Synthetic dataset: unsmoothed results for mean and 99\% confidence interval of prediction width with respect to 1023 missingness patterns.}
\label{synthetic_width_unsmoothed}
\end{center}
\vskip -0.3in
\end{figure}
\begin{figure}[ht]
\begin{center}
\centerline{\includegraphics[width=0.8\textwidth]{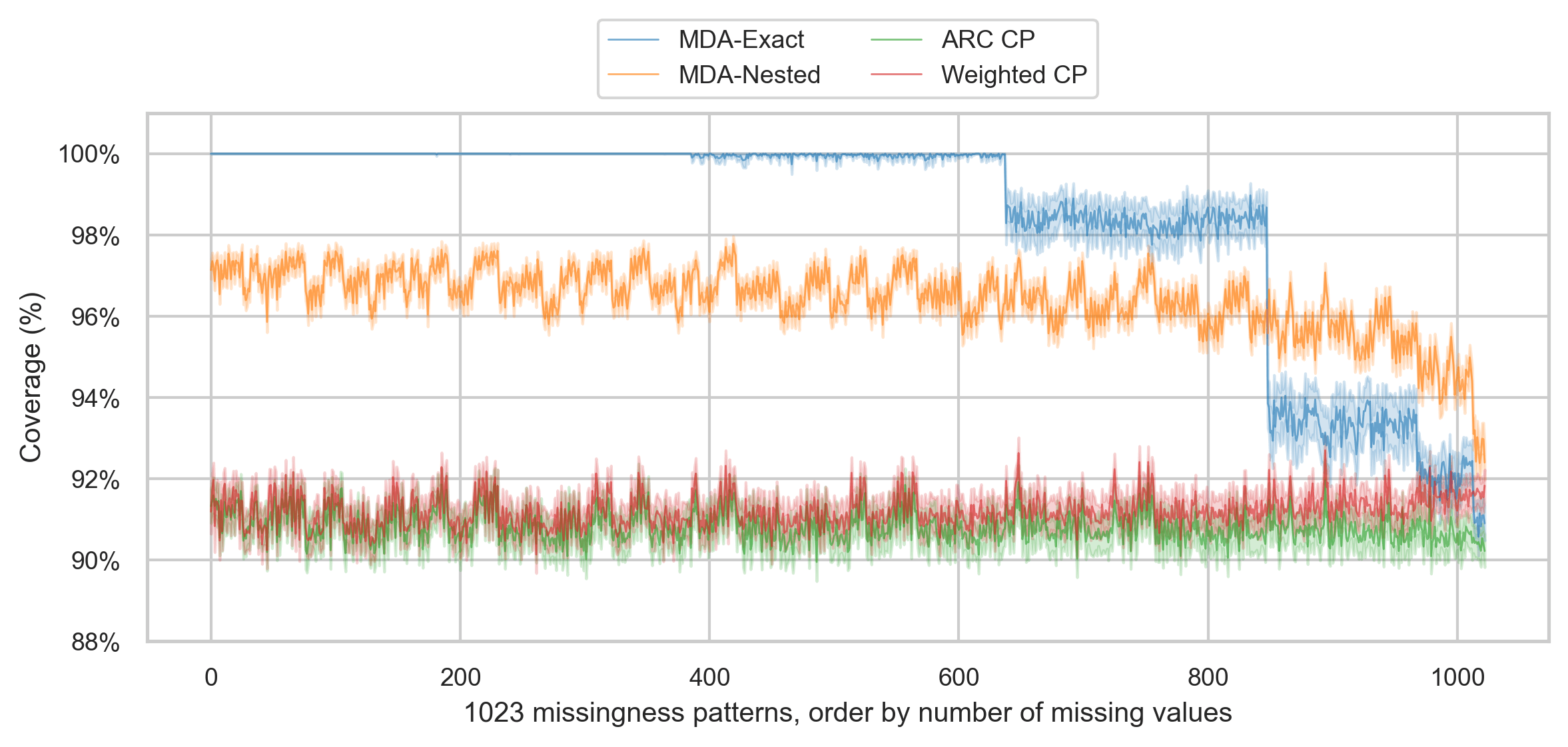}}
\caption{Synthetic dataset: unsmoothed results for mean and 99\% confidence interval of prediction coverage with respect to 1023 missingness patterns.}
\label{synthetic_coverage_unsmoothed}
\end{center}
\vskip -0.1in
\end{figure}
\subsection{Coefficient vector $\beta$ for the 100-dimensional synthetic experiment}
\label{appendix:100beta}
In the 100-dimensional synthetic experiment of Section~4.2, we fix the linear model
coefficients as
$
\beta = 
$(
1, 2, -1, 3, -0.5, -1, 0.3, 1.7, 0.4, -0.3,
-2.84, 0.91, 2.77, -0.68, 1.53, -1.94, 0.44, -2.21, 2.05, -0.17,
-1.46, 0.23, 1.88, -2.63, 2.41, -0.82, 0.71, -1.12, 2.58, -0.49,
-2.05, 1.37, -0.94, 2.63, -1.59, 0.52, -2.77, 1.96, -0.08, 2.24,
-1.23, 0.35, 2.49, -2.36, 1.14, -0.71, 2.81, -1.76, 0.62, -2.48,
1.79, -0.27, 2.92, -1.65, 0.84, -2.58, 2.18, -0.41, 1.25, -2.93,
2.67, -0.56, 1.48, -2.12, 0.19, 2.32, -1.88, 0.77, -2.69, 2.55,
-0.63, 1.61, -2.44, 0.58, 2.86, -1.31, 0.05, -2.29, 1.92, -0.14,
2.37, -1.57, 0.49, 2.99, -2.72, 1.08, -0.35, 2.14, -1.99, 0.66,
-2.11, 0.88, -1.45, 2.63, -0.72, 1.57, -2.33, 0.12, 2.48, -1.08
), as shown in the released code.